\documentclass[12pt]{article}

\usepackage{amssymb,amsmath,amsthm} 
\usepackage{graphicx, rotating}      
\RequirePackage{multirow}            
\usepackage{url}                     
\usepackage{algpseudocode}           
\usepackage{algorithm}               
\usepackage{algorithmicx}            
\usepackage{subcaption}              
\usepackage{paralist}                
\usepackage{mathrsfs}                
\usepackage{arydshln}                
\usepackage{booktabs}                
\usepackage{mathtools}               
\usepackage{natbib}        
\usepackage[normalem]{ulem}          
\RequirePackage[colorlinks,citecolor=blue,urlcolor=blue]{hyperref} 

\newtheorem{theorem}{Theorem}

\newtheorem{remark}{Remark}

\DeclareGraphicsExtensions{.jpg,.pdf,.mps,.png} 

\author{
  Jialai She
}
\date{} 

\title{Beyond Additivity: Sparse Isotonic Shapley Regression toward Nonlinear Explainability}

\begin{document}

\maketitle
\begin{abstract}
Shapley values, a gold standard for feature attribution in Explainable AI, face two primary challenges. First, the canonical Shapley framework assumes that the worth function is additive, yet real-world payoff constructions---driven by non-Gaussian distributions, heavy tails, feature dependence, or domain-specific loss scales---often violate this assumption, leading to distorted attributions. Second, achieving sparse explanations in high-dimensional settings by computing dense Shapley values and then applying ad hoc thresholding is prohibitively costly and risks inconsistency.
We introduce Sparse Isotonic Shapley Regression (SISR), a unified nonlinear explanation framework. SISR simultaneously learns a monotonic transformation to restore additivity---obviating the need for a closed-form specification---and enforces an L0 sparsity constraint on the Shapley vector, enhancing computational efficiency in large feature spaces. Its optimization algorithm leverages Pool-Adjacent-Violators for efficient isotonic regression and normalized hard-thresholding for support selection, ensuring ease in implementation and global convergence guarantees.
Analysis shows that SISR recovers the true transformation in a wide range of scenarios and achieves strong support recovery even in high noise. Moreover, we are the first to demonstrate that irrelevant features and inter-feature dependencies can induce a true payoff transformation that deviates substantially from linearity.  Extensive experiments in regression, logistic regression, and tree ensembles demonstrate that SISR stabilizes attributions across payoff schemes and correctly filters irrelevant features; in contrast, standard Shapley values suffer severe rank and sign distortions. By unifying nonlinear transformation estimation with sparsity pursuit, SISR advances the frontier of nonlinear explainability, providing a theoretically grounded and practical attribution framework.
\end{abstract}
\section{Introduction and Motivation}\label{sec:intro}
Let \( F = \{1, 2, \ldots, p\} \) denote the set of \( p \) features, and let \( \nu: 2^F \rightarrow \mathbb{R} \) be the {characteristic function or payoff function, with \( \nu(A) \) representing the   contribution or {worth} generated by a subset \( A \subseteq  F \) of features working together  (often referred to as a \emph{coalition} in game theory).
A central question in economics and cooperative game theory is how to fairly allocate the value of a coalition to its individual members.    The \emph{Shapley value} \citep{shapley1953value}, a concept from Nobel laureate Lloyd Shapley, offers a theoretically grounded solution by assigning payoffs according to each member's average marginal contribution across all possible subsets.

In this paper, we denote the Shapley value for feature \( j \) by \( \beta_j \) for \( 1 \leq j \leq p \), quantifying the fair share or importance of feature \( j \);  for a subset \( A \subseteq F \)   we define \( \beta_A \) as the vector \( [\beta_j]_{j \in A} \in \mathbb{R}^{|A|} \).
For brevity, we also write \( \nu_A \) as shorthand for \( \nu(A) \) for any subset \( A \subseteq F \). Shapley values establish a connection between the payoff function \( \nu(A) \) and the underlying model parameters \( \beta_A \). To make this dependence explicit, we introduce a  function \( V(\beta_1, \ldots, \beta_p; A) \), also denoted by \( V_A(\{ \beta_j \}_{j \in A}) \),  characterizing the deterministic, \emph{noise-free} contribution associated with subset \( A \):
\begin{align}
\nu_A \sim V_A(\{ \beta_j \}_{j \in A}), \label{nu-V}
\end{align}
where \( \sim \) denotes approximate equality up to noise, a convention  adopted throughout the paper. 

In recent years, Shapley values have attracted substantial attention in machine learning, particularly in the field of \emph{Explainable AI} (\textbf{XAI}) \citep{ancona2019explaining}. While assessing variable importance in simple regression is straightforward using traditional tools like $T$-tests and $p$-values, this task becomes a formidable challenge for the complex, ``black-box" models now widely used to analyze sequential data in economics and finance. For sophisticated models---ranging from \textit{tree-based ensembles} like random forests and boosted trees to \textit{deep neural networks} like Long Short-Term Memory (LSTM) networks---standard inference methods are no longer applicable, making interpretation  notoriously difficult.

Shapley values provide a model-agnostic framework for quantifying feature importance. This is applied in two main ways: local explanations, which explain a single prediction (e.g., SHAP by  \cite{Lundberg2017}), and global explanations, which explain the model's overall behavior (e.g., SAGE by \cite{covert2020understanding}). Our work focuses on the global setting, where the goal is to find a single, interpretable set of importance values for the entire model.
Specifically, for a prediction model \( f(x) \), where \( x \in \mathbb{R}^p \), researchers first design a payoff function \( \nu_A \) over subsets \( A \subseteq F \) to quantify the model's global performance when using only the features in \( A \). This reframes the explanation task as a ``\textbf{credit allocation}'' problem, enabling the use of Shapley values to quantify feature importance, and perhaps more importantly, to construct interpretable \textit{surrogate models} based on restricted feature sets.
The resulting additive structure of individual feature contributions is the \textit{very} property  appealing for interpretability, as it provides an intuitive, linear explanation of an intricate model's behavior,  a primary goal in XAI.

However, standard Shapley-based methods also face several limitations that restrict their practical utility in complex modeling scenarios.
\\

\textbf{(i) Moving beyond additive frameworks:}
Given a prediction model, various methods have been proposed  to construct the payoff function $\nu_A$ for Shapley-value analysis. (a)
A fundamental approach involves retraining the model on every subset of features $A$ and defining $\nu_A$ based on the reduction in statistical accuracy (such as  $R^2$  in regression) \citep{Lipo2001}. However, this exhaustive procedure may be computationally prohibitive for modern  AI models due to its exponential cost. (b) To circumvent retraining,   SAGE \citep{covert2020understanding} provides an efficient alternative: it keeps the trained model fixed and quantifies the expected loss increase when certain features are made unavailable.   The approach marginalizes over missing inputs---conditionally in theory and, for scalability, interventionally in practice.
(c) In contrast, SHAP and TreeSHAP \citep{Lundberg2017,Lundberg2020} define local payoffs  for each instance by marginalizing absent features and then derive global importance by aggregating the resulting local Shapley attributions. (d) Other global variants, such as Sobol-Shapley indices   \citep{owen2014sobol}
and derivative-based formulations \citep{Duan2025}, approximate risk or variance decomposition under specific assumptions (feature independence, distributional priors, or model smoothness), often motivated by numerical sensitivity analysis rather than prediction risk.
Once $\nu_A$ is constructed, researchers often mechanically apply the Shapley formula to compute feature attributions.

However, the theoretical justification for Shapley values relies on several foundational   ``{axioms'}'---efficiency, symmetry, linearity, and nullity \citep{shapley1953value}---which   are not easily testable and are \emph{rarely} validated in practice. In particular, Shapley's framework implicitly assumes an \textbf{additive structure} \citep{Lundberg2017}:
\begin{align}
\nu_A \sim \sum_{j\in A} \beta_j \quad \mbox{ or } \quad V_A(\{ \beta_j \}_{j \in A}) = \sum_{j\in A} \beta_j. \label{shapleyadditive}
\end{align}
But the so-called additive feature attribution is not guaranteed to hold in real-world constructions of coalition values.
 For example,  we can reformulate the abstract Shapley axioms and principles  into a multivariate Gaussian assumption   (cf. Section~\ref{sec:method}), but many of the constructions mentioned previously are prone to violating this assumption due to \textbf{non-Gaussian} characteristics such as bounded ranges, heavy tails, and skewness.
     In particular,   \cite{fryer2021shapley}   recently proposed a realistic ``taxicab'' payoff   defined by a \textit{winner-takes-all} dynamic  that is  in stark contrast to \eqref{shapleyadditive}:
\begin{align}
 V_A(\{ \beta_j \}_{j \in A})=\max_{j\in A} \beta_j, \ \forall A\subseteq  F.
\end{align}
     Such nonlinear relations are prevalent in    applications but fundamentally violate the additive model underpinning standard Shapley value estimation.
\\

\textbf{(ii) Embedding sparsity into value attribution:}  In many real-world applications with a large number of features, a substantial proportion contribute only negligibly---or are effectively irrelevant---to the overall outcome, making them unnecessary to explain in practice \citep{strumbelj2014explaining,covert2021explaining}.  Exploiting the    structural parsimony can enhance both  statistical accuracy and  interpretability of Shapley values.   In implementation, leveraging sparsity helps to  reduce  iteration complexity, thanks to a substantially smaller   effective model size, along with  mitigating communication costs and storage requirements  in high-dimensional settings.

However, existing approaches adopt a \textbf{greedy} strategy to achieve sparsity  and have significant drawbacks. Many methods first compute dense Shapley values for the \emph{full} model and then apply post-hoc ranking or thresholding  \citep{Cohen2007,JOTHI2021103,fryer2021shapley,Au2022}.
For large $p$, such multi-step procedures are not only  inefficient  but may also fail to provide faithful explanations or meaningful   selection  (see, e.g., \cite{covert2021explaining,slack2020fooling,pmlr-v127-ma20a}).
An alternative class of older, less efficient approaches resorts to an $\ell_1$-penalty   \citep{Lundberg2017,ribeiro2016should}, but requires cumbersome parameter tuning and induces unwanted shrinkage on the attribution values, which can distort their magnitude.
This often necessitates a multi-step re-fitting procedure to correct for the shrinkage, undermining the goal of a unified estimation. More fundamentally, this entire approach relies on the $\ell_1$-norm's ability to select the correct features, an inherent drawback as it often fails to recover the true support, especially in the presence of correlated features \citep{zhang2010nearly}.

To the best of our knowledge, \emph{no} widely adopted framework integrates direct, shrinkage-free sparsity control as an intrinsic property into Shapley-value estimation, let alone in the context of an unknown transformation.
These challenges underscore the need for a unified approach that \textit{simultaneously} enforces sparsity and ensures coherent Shapley-based attributions.
\\

This paper aims to develop a novel nonlinear explanation framework for applying the Shapley mechanism in a way that simultaneously   aligns individual feature contributions with appropriately transformed worths across all subsets, and  promotes sparsity by eliminating irrelevant features to enhance both computational efficiency and statistical accuracy. The contributions of our work are as follows:
\begin{itemize}
    \item  Our research is the {first to  demonstrate} that common factors such as the {presence of irrelevant features} and {inter-feature dependencies} can induce a  payoff transformation that {deviates substantially from linearity}, even when using standard payoff constructions (e.g., $R^2$-based worths). This finding underscores the  need for nonlinear explainability frameworks.

    \item We propose Sparse Isotonic Shapley Regression (\textbf{SISR}), the first framework to \textit{jointly} address payoff non-additivity and attribution sparsity. By learning a monotonic transformation and enforcing an $\ell_0$  constraint simultaneously, our integrated approach overcomes the limitations of  ad-hoc methods.

    \item SISR   learns the   transformation of payoffs {{without requiring a predefined analytical form}}. This is achieved through efficiently leveraging the  {Pool-Adjacent-Violators} algorithm, allowing the model to adapt to diverse  real-world payoff structures.


    \item   The optimization algorithm developed for SISR features simple, {closed-form} updates and is accompanied by {global convergence guarantees}.
The incorporation of sparsity improves computational efficiency.
    \item  Through extensive experiments across various datasets   and payoff schemes, we show that SISR significantly stabilizes feature attributions and correctly identifies relevant features, mitigating the severe rank and sign distortions often observed with standard Shapley value applications.
\end{itemize}

Our goal with SISR is not to abandon the interpretability of additivity, but rather to restore it. While other valuable frameworks  explicitly model  higher-order feature interactions (cf. Section \ref{sec:ext}), this paper proposes a distinct alternative that seeks to \textit{learn} a principled monotonic transformation that maps the payoff function \textit{back} to a domain where a simple, additive main-effect structure holds. This preserves the interpretability of the original Shapley framework while robustly handling payoff constructions driven by non-Gaussian distributions and domain-specific loss scales  that violate its core assumptions.

The rest of the paper is organized as follows. Section~\ref{sec:method} proposes a novel  Sparse Isotonic Shapley Regression    model to address challenges related to domain adaptation and high dimensionality. In Section~\ref{sec:comp}, an optimization-based algorithm is developed to address the functional challenge and the nonsmooth sparsification, with  established theoretical guarantees. Section~\ref{sec:experiments}  provides valuable data-driven insights drawn from  experiments in various scenarios. Extensions are given in Section \ref{sec:ext}. We conclude in Section \ref{sec:summ}.

\section{Proposed Method}\label{sec:method}
The Shapley axioms and principles have been interpreted by economists in various ways \citep{algaba2019handbook}. Here, we recast the Shapley framework as a statistical assumption on the data-generating process. To begin, let us revisit a motivating \textit{weighted least squares} formulation of Shapley value estimation as derived in \cite{Lundberg2017}, echoing earlier developments in econometrics \citep{Charnes1988}:
\begin{align}
\begin{split}
\min_{\beta\in \mathbb R^p, c\in\mathbb R} &\sum_{A \in 2^F,\, A \neq \emptyset,\, A \neq F} w_{\text{\tiny SH}}(A) \left( \nu_A - \sum_{j \in A} \beta_j - c \right)^2 \\
&\text{subject to} \quad c = \nu_\emptyset, \quad c + \sum_{j=1}^p \beta_j = \nu_F,
\end{split}
\label{wls-shapley}
\end{align}
where the Shapley weights are given by
\begin{align}
w_{\text{\tiny SH}}(A) = \frac{p-1}{{p \choose |A|} |A| (p - |A|)}.
\label{shapleyweights}
\end{align}
Perhaps surprisingly, it can be shown that the optimal solution \( \hat{\beta} \) to \eqref{wls-shapley} recovers the exact Shapley values  \citep{Lundberg2017}, which are traditionally derived based on the concept of \textit{marginal  contributions} across all possible feature coalitions.

If we define
\begin{align}
w_{\text{\tiny SH}}(\emptyset) = +\infty, \quad w_{\text{\tiny SH}}(F) = +\infty \label{bigweights}
\end{align} as an extension of \eqref{shapleyweights}, then  \eqref{wls-shapley} can be written as
\begin{align*}
 \min_{\beta,c} &\sum_{A\in 2^{F}}  w_{\text{\tiny SH}}(A) (\nu_A - \sum_{j\in A} \beta_j -c)^2
\end{align*}
where $A$ can take any subset of the power set $2^F$. Note that when $A = \emptyset$, $ \sum_{j\in A} \beta_j=0$ by convention and   $\hat c = \nu_\emptyset$. It is thus convenient to define the \emph{baseline-adjusted} coalition values: \begin{align} \nu_A^{c} = \nu_A - \nu_\emptyset, \quad \forall A \in 2^F, \label{nucenter} \end{align} which saves one parameter in the subsequent optimization: \begin{align*} \min_{\beta} \sum_{A\in 2^{F}} w_{\text{\tiny SH}}(A) \left( \nu_A^c - \sum_{j\in A} \beta_j \right)^2. \end{align*}For notational simplicity, we will   write `$\nu_A^c$' as just `$\nu_A$', assuming that all $\nu$ values have been   properly shifted  according to   \eqref{nucenter} unless otherwise specified.

It is helpful to reinterpret the weighted least squares formulation of Shapley values as a probabilistic model:
\begin{align}
\begin{split}
&\nu_A \sim \mathcal N(\mu_A,  \sigma_A^2) \\
& \mu_A = \sum_{j\in A} \beta_j^*, \\
&\sigma_A^2 \propto {p\choose |A|} |A| (p - |A|)\, (\propto \frac{1}{w_{\text{\tiny SH}}(A)}),
\end{split}\label{shapleygaussmodel}
\end{align}
and all $\nu_A$'s are independent. Here, $\beta_j^*$ denotes the true Shapley value for the $j$th feature.

By reformulating the Shapley axioms and principles into  assumption    \eqref{shapleygaussmodel}, we gain insight into why numerous payoff  functions may not meet the  model criteria. Indeed, due to issues such as range constraints, skewness, heavy tails, and heterogeneity, it is natural to question the appropriateness of the   multivariate Gaussianity across different definitions  of coalition values.

In our view, one viable solution is to apply a transformation that promotes  Gaussianity. Let's consider an alternative   Shapley value model in a \textit{transformed domain}:
\begin{align}
T(\nu_A) \sim \mathcal{N}( \sum_{j \in A} T(\beta_j^*), \, \sigma_A^2 ),
\label{TShap}
\end{align}
where \( T(\cdot) \) is an unknown transformation.  Under this model,
\begin{align}
\mathbb{E}\left[ T(\nu_A) \right] = \sum_{j \in A} T(\beta_j^*),
\end{align}
which defines a ``$T$-additive'' framework for nonlinear settings.
To model this structure, we propose a new Shapley framework termed  {Functional Shapley Regression}, which jointly estimates \( \beta \) and \( T(\cdot) \) by solving
\begin{align}
\min_{\beta, T(\cdot)} \sum_{A \in 2^{F}} w_{\text{\tiny SH}}(A) \Big\{ T(\nu_A) - \sum_{j \in A} T(\beta_j) \Big\}^2 \  \text{subject to} \  \beta \in \mathcal{C}, \, T(\cdot) \in \mathcal{T},
\label{functional-shapley}
\end{align}
where the objective minimizes the Shapley-weighted sum of squared differences between the transformed coalition values \( T(\nu_A) \) and the transformed linear sum \( \sum_{j \in A} T(\beta_j) \) over all subsets \( A \in 2^{F} \). Here, we use  \( \mathcal{C} \subseteq \mathbb{R}^p \) to denote the constraint set for \( \beta \), and \( \mathcal{T} \) to denote the class of admissible transformation functions.
By the notational convention for \( A = \emptyset \), \eqref{functional-shapley} automatically enforces $$ T(0) = 0 ,$$ corresponding to   \( T(\nu_\emptyset) = 0 \) (recall all $\nu_A$ have been centered).

The remark below illustrates that  our framework, in contrast to the common additive main-effect model   (see, e.g.,  \cite{Lundberg2017}), accommodates a broader range of multivariate payoff structures. By learning a transformation to restore the Shapley framework's underlying statistical assumptions (Gaussianity), the resulting $T^{-1}$-sum-$T$ structure in \eqref{VinT}     enables \textit{nonlinear} explainability.

\begin{remark}[\textbf{Univariate $T$-Mappings} for \textbf{Multivariate Structure}]
Introducing a \emph{univariate} transformation \( T(\cdot) \)   enables a remarkably rich class of models capable of capturing   complex \emph{multivariate} relationships between \( \nu_A \) and \( \{ \beta_j : j \in A \} \), well beyond the standard additive form.

Specifically, under the \( T \)-transformed model \eqref{TShap},  assuming the existence of the inverse transformation    \( T^{-1} \)   and using the notation $V_A$ (cf. \eqref{nu-V}) we have
\begin{align}
 V_A (\{\beta_j\}_{j \in A}) =  T^{-1}\Big( \sum_{j \in A} T(\beta_j) \Big) \quad \emph{ or } \quad \nu_A\sim T^{-1}\Big( \sum_{j \in A} T(\beta_j) \Big).\label{VinT}
\end{align}
If \( T \) is a nondegenerate linear map like the identity map,   the model reduces to the conventional additive Shapley game, where the coalition value $\nu_A$ in \eqref{VinT} is essentially  a simple sum of individual contributions. However,  a general transformation lends the  $T^{-1}$-$\Sigma$-$T$  multivariate structure of \eqref{VinT} the flexibility to  model a broad range of application domains. By \emph{learning} the transformation from the data, our framework acts as a robust generalization of the conventional model, rather than imposing a forced, arbitrary transformation.

For instance, consider a monomial transformation \( T(x) = |x|^d \) for some \( d > 0 \),  which, under the mild assumption of non-negative payoffs, induces a multivariate ``\( d \)-norm" relationship:
\[
 V_A (\{\beta_j\}_{j \in A})= \big( \sum_{j \in A} |\beta_j|^d \big)^{1/d} = \| \beta_A \|_d.
\]
Varying the degree \( d \) recovers a spectrum of geometric structures, e.g.,
\begin{itemize}
    \item[(i)] \( d = 1 \): the \(\ell_1\)-norm polytope, \( V_A (\{\beta_j\}_{j \in A})= \sum_{j \in A} |\beta_j| \);
    \item[(ii)] \( d = 2 \): the \(\ell_2\)-norm ball, \( V_A (\{\beta_j\}_{j \in A}) =  \big( \sum_{j \in A} \beta_j^2  \big)^{1/2} \);
    \item[(iii)] \( d \to \infty \): the \(\ell_\infty\)-norm cube, \( V_A (\{\beta_j\}_{j \in A})= \max_{j \in A} |\beta_j| \).
\end{itemize}
In particular, under nonnegativity constraints (\( \nu_A \geq 0 \), \( \beta_j \geq 0 \)), the \(\ell_\infty\) case corresponds to the \emph{winner-takes-all} mechanism as first noted in \cite{fryer2021shapley}, where the coalition value  is dominated by the largest individual contribution (practically, monomial transformations with   large degrees \( d \) can closely approximate  such behavior). This is  motivating, as the examples here are highly nonlinear and  incompatible with a linear Shapley game. Yet, with a  univariate transformation, they can be incorporated into the   $T$-Shapley framework. Additional examples are the exponential form   $T(x) = \exp(x)-1$  and the odds form $T(x) = \Phi(x)/(1-\Phi(x))$ with $\Phi$ a distribution function of a continuous random variable.

In sum, an appropriately chosen   \( T(\cdot) \) establishes  a versatile \emph{nonlinear} modeling mechanism that  enhances the additive expressiveness of Shapley values for  XAI. Another advantage of the proposed approach is that it \emph{\textbf{bypasses}} the need for a predefined analytic transformation, instead learning it directly from the data (cf. Section \ref{sec:comp}). This novel capability of ``\emph{learning to be additive}' to   recover a linear, main-effect Shapley structure is a key contribution of the framework. \end{remark}


In this paper, we focus on a  specific instance   of \eqref{functional-shapley}, referred to as the ``{{\textbf{S}parse \textbf{I}sotonic \textbf{S}hapley \textbf{R}egression}}'' (\textbf{SISR}):
\begin{align}
\begin{split}&
\min_{\beta, T(\cdot)} \sum_{A\in 2^{F}} w_{\text{\tiny SH}}(A) \Big\{ T(\nu_A )- \sum_{j\in A} T(\beta_j) \Big\}^2 \\
&\mbox{ subject to } \| \beta\|_0\le s, T\in \mathcal M,  \sum_{j = 1}^p  (T  (\beta_j))^2    =  1,
\end{split}
\label{spaiso-shapley}\end{align}
where  $\mathcal M$   denotes the class of strictly increasing functions and $1\le s \le p$ specifies the user-defined upper bound on the true model sparsity. Notably, this objective  is   defined on the $T(\nu_A)$ scale, as this is the domain where the Gaussian error assumption holds (cf. \eqref{TShap}); applying a quadratic loss to the original $\nu_A$ scale would be statistically inconsistent.
\eqref{spaiso-shapley} incorporates three critical modeling considerations.
\paragraph*{Monotonicity.}
We impose a monotonicity constraint on \( T(\cdot) \) to preserve the relative ordering of feature importance values:
\[
\beta_i \geq \beta_j \quad \Rightarrow \quad T(\beta_i) \geq T(\beta_j).
\]
This ensures that the learned transformation respects the relative contribution levels of individual features.  The structure closely resembles \textbf{isotonic regression} \citep{robertson1988order}, which seeks a weighted least squares fit under monotonicity constraints and has widespread applications in    psychometrics,  epidemiology, yield curve estimation, risk modeling and credit scoring  studies. Compared with enforcing smoothness in    $T$, our monotonicity approach avoids any need for a basis expansion or other parametric representation. Pursuing $T$ reinterprets the data in a transformed domain where feature contributions recover an additive Shapley structure.

\paragraph*{Normalization.} A  normalization is   imposed on the transformed feature contributions, $\sum_{j = 1}^p (T(\beta_j))^2 = 1.$
This prevents degeneracy (e.g., trivial solutions such as \( T \equiv 0 \)) and anchors the scale of the model.  An appealing feature of~\eqref{spaiso-shapley} is its invariance to the overall scaling of $\{\nu_A\}$, and the normalization constant is fixed at 1 without loss of generality.   Moreover, Section~\ref{sec:comp} will show that imposing such a spherical constraint  yields  computational benefits, enabling a closed-form solution for the attribution update  and improving    implementability.

\paragraph*{Sparsity.}
\eqref{spaiso-shapley} directly incorporates \textbf{sparsity} into the Shapley estimation process. Rather than relying on multi-step methods that first estimate a dense Shapley vector and then rank features \citep{slack2020fooling}, the formulation   constrains the support of \( \hat{\beta} \)  while pursuing the  transformation   {during} the iterative optimization process (cf.  Algorithm \ref{Alg:SpaIso}). This \textit{unified} treatment ensures that sparsity, domain adaptation, and Shapley coherence are achieved simultaneously, avoiding  inconsistencies  by post hoc selection.
The  popular  \(\ell_1\)-penalty $\lambda \sum |\beta_j |$  requires cumbersome    $\lambda$-tuning   and induces unwanted shrinkage. For example, it is generally unclear \textit{a priori} how many nonzero coefficients result from a particular choice of  $\lambda$.  But our $\ell_0$ constraint offers direct control over model sparsity and is entirely shrinkage-free, avoiding the attribution-distorting bias of $\ell_1$-type methods and the need for multi-step re-fitting procedures. It  further remedies a well-documented limitation of  $\ell_1$ selection, which often fails to recover the true support in the presence of even moderately correlated features  \citep{Zhao2006,zhang2010nearly}. These properties make \eqref{spaiso-shapley} particularly appealing for applications such as bioinformatics, where practitioners frequently require a fixed number of interpretable predictors. For cases where $s$ must be chosen, we find the RIC criterion \citep{Foster94} to be effective within our Shapley framework.      \\

Before concluding this section, we introduce a \textit{reparameterization} trick that proves   beneficial for both modeling and computation. Define
\begin{align}
\gamma_j = T(\beta_j).
\end{align}
Since   \(T\) is strictly increasing and \(T(0) = 0\) (the loss would become infinite if \(T(0) \ne 0\)),    \eqref{spaiso-shapley} can be rewritten as \begin{align}
 \min_{\gamma, T(\cdot)} \quad  \sum_{A \in 2^{F}} w_{\text{\tiny SH}}(A)  ( T(\nu_A ) - \sum_{j \in A} \gamma_j  )^2  \text{ s.t. }    \| \gamma \|_0 \le s, \| \gamma \|_2 = 1,  T \in \mathcal{M}.
\label{prob-reparam}
\end{align}
The corresponding model assumption   is thus  $
  T^*(\nu_A) \sim \mathcal{N} ( \sum_{j \in A} \gamma_j^*,\ \sigma_A^2  )$   for all $ A \in 2^{F}, 
$ 
where the genuine transformation function $T^*$ is monotonic with $T^*(0) = 0$, and   $ \nu_A$ are assumed to be independent across different subsets $A$. The \emph{starred} quantities represent the underlying statistical truth of interest to estimate. Assume   $\gamma^* \in \mathbb{R}^p$  satisfies   $\| \gamma^* \|_0 \leq s^*$ and $\| \gamma^* \|_2 = 1$ with $1\le s^*\le p$ and $s$  is specified as an upper bound on $s^*$. After estimating \( \hat{\gamma} \) and \( \hat{T} \) from \eqref{prob-reparam}, one can recover the   $\beta$-scores by applying the inverse transformation
$
\hat{\beta}_j = \hat{T}^{-1}(\hat{\gamma}_j).
$
This reconstructs the multivariate relationship between \( \nu_A \) and the set of feature contributions in the original scale, yielding   $\nu_A\approx \hat T^{-1}(\sum_{j\in A} \hat T(\hat \beta_j))$, to  offer interpretable Shapley-based attributions.

\section{Optimization Algorithm}\label{sec:comp}
The optimization of SISR involves two main challenges: (i) a functional estimation component, and (ii) a combinatorial sparsity constraint coupled with a nonconvex normalization constraint. We show that the functional challenge can be   addressed by a discretization technique, which, rather than introducing an approximation, preserves full equivalence.   To handle the two constraints  on $\gamma$, we develop a  surrogate function framework. These efforts  lead to an iterative procedure that combines the pool-adjacent-violators  with a normalized hard thresholding. Each step  has implementation ease and  the sparse structure  ensures that the overall algorithm remains  efficient  in high-dimensional settings.

First, since \(T(\cdot)\) is only evaluated at the observed values \(\nu_A\) in the objective function, we   ``discretize'' \eqref{prob-reparam} by introducing the vector $$t =\big [T(\nu_A)\big]_{A \in 2^F} \in \mathbb{R}^{2^p}.$$ In defining this vector, one should fix a specific order over subsets \(A \subseteq F\); we follow the conventional lexicographic binary ordering to arrange the entries of \(t\).
Correspondingly, we define $$
\nu = \big[\nu_{ A}\big]_{A \in 2^F}   \in \mathbb{R}^{2^p},\quad \delta = \Big[\sum_{j \in A} \gamma_j\Big]_{A \in 2^F}   = Z\gamma\in \mathbb{R}^{2^p},  $$ where \(Z \in \mathbb{R}^{2^p \times p}\) is the ``incidence matrix'' indicating which features are active in each subset \(A\), aligned with the same ordering used for \(t\). Henceforth, we  also  write $\nu_i$  (and likewise $\delta_i$) to denote the entry corresponding to the $i$th subset. Additionally, we introduce the diagonal weight matrix
\begin{align}
W = \text{diag}\{w_{\text{\tiny SH}}(A)\}_{A \in 2^F}. \label{Wdef}
\end{align}

With this notation in place, we study the following optimization problem:
\begin{align}
\begin{split}
\min_{\gamma \in \mathbb{R}^p,\; t \in \mathbb{R}^{2^p}} \quad & \frac{1}{2} (t - \delta)^\top W (t - \delta) \\
\text{subject to} \quad & \delta = Z\gamma,\quad \| \gamma \|_0 \le s,\quad \| \gamma \|_2 = 1, \\
& t_i \le t_j \quad \text{for all } (i,j) \in E(\nu)= \{(i, j) : \nu_i \le \nu_j\},
\end{split}
\label{prob-opt}
\end{align}
where \(E \) encodes the pairwise ordering constraints induced by $\nu$, due to the monotonicity of the transformation \(T\). This formulation replaces strict monotonicity with a non-decreasing constraint, a mild adjustment that facilitates numerical implementation.
 In the remainder of the section, we design a two-block alternating optimization algorithm. 

First, with \(\delta\) fixed, the optimization over \(t\) corresponds to   the   (weighted) \textit{isotonic regression}
\begin{align}
\min_{  t \in \mathbb{R}^{2^p}} \quad & \frac{1}{2} (t - \delta)^\top W (t - \delta)
\text{ subject to }   \  t_i \le t_j \ \text{for all } (i,j) \in E, \label{subopt-iso}
\end{align}
 where the goal is to obtain a monotonic fit to \(\delta\) under a weighted squared-error loss defined by \(W\). The problem can be solved using any standard Quadratic Programming ({QP}) solver, but it is more efficiently handled by the Pool-Adjacent-Violators Algorithm (\textbf{PAVA}) \citep{JSSv032i05}, which leverages the structure of the monotonicity constraints for improved computational performance.

Next, we focus on  the $\gamma$-optimization.

\begin{theorem}\label{th:normHT}
Let \(\mathcal  H(\cdot; s)\) denote the \emph{hard-thresholding} operator associated with cardinality \(s\), defined as follows: for a vector \(y \in \mathbb{R}^p\), \(\mathcal  H(y; s) = z\) where \(z_i = y_i\) if \(|y_i|\) is among the \(s\) largest entries of \(|y_1|, \ldots, |y_p|\), and \(z_i = 0\) otherwise, and the  \emph{normalized hard-thresholding} operator  \({\mathcal  H}^\circ (y; s)={\mathcal H}(y; s)/\|{\mathcal H}(y; s)\|_2\) if $\|{\mathcal H}(y; s)\|_2 \ne  0$.
Then, for the optimization problem with $y \ne \vec 0$,  $1\le s \le p$,
\[
\min_\beta \frac{1}{2} \| y - \beta \|_2^2 \quad \text{subject to} \quad \|\beta\|_0 \le s,\ \|\beta\|_2 = 1,
\]
the vector obtained by normalized hard-thresholding, $$\hat \beta  ={\mathcal H}^\circ(y; s)= \frac{\mathcal H(y; s)}{\| \mathcal  H(y; s) \|_2}$$ is a global optimizer.
\end{theorem}
\begin{proof}
Let   $A\subseteq \{1, \ldots p\}$ and assume $\beta_{A^c} = 0$ and $\| \beta\|_2 =1$. Because
\begin{align*}
 \| y - \beta \|_2^2 = & \| y \|_2^2 + \|\beta\|_2^2 - 2\langle y, \beta\rangle  \\= & \| y \|_2^2 + 1 - 2\langle y, \beta\rangle \\
 = & \| y \|_2^2 + 1 - 2\langle y_A, \beta_A\rangle \\
  \ge & \| y \|_2^2 + 1 - 2 \|y_A\|_2 \| \beta_A\|_2 =  \| y \|_2^2 + 1 - 2 \|y_A\|_2,
\end{align*}
where we used the Cauchy-Schwarz inequality and the equality is achieved at $\beta_A = y_A/\| y_A\|_2$.

Therefore, $\min_{\beta: \beta_{A^c} =0, \|\beta\|_2 = 1}   \| y - \beta \|_2^2 = \| y \|_2^2 + 1 - 2 \|y_A\|_2 $ for any $A: |A| = s$. Minimizing over $A$ gives an index set corresponding to  the \(s\) largest entries of \(|y_1|, \ldots, |y_p|\),  thereby the normalized hard thresholding operator ${\mathcal H}^\circ(y; s)$.
\end{proof}

We are now ready to develop an iterative algorithm for updating \(\gamma\) with \(t\) held fixed. Define the objective function
\[
l(\gamma) = \frac{1}{2} (  Z\gamma - t)^\top W (  Z\gamma - t). 
\]
A straightforward calculation yields the gradient:
\begin{align}
\nabla l(\gamma) = Z^\top W (Z\gamma - t).\label{gradform}
\end{align}
To facilitate optimization, we construct a new ``surrogate function'':
 \begin{align}
g(\gamma, \gamma^-) = l(\gamma^-) + \langle \nabla l(\gamma^-), \gamma - \gamma^- \rangle + \frac{\rho}{2} \| \gamma - \gamma^- \|_2^2. \label{surrofunc}
\end{align}
where   \(\rho > 0\) should be properly large (cf. Theorem \ref{th:gammaoptalg}).
Define an iterative scheme:
\[
\gamma^{(k+1)} = \arg\min_\gamma\; g(\gamma, \gamma^{(k)}) \quad \text{subject to } \|\gamma\|_0 \le s,\; \|\gamma\|_2 = 1.
\]
Using Theorem~\ref{th:normHT}, the update step admits a closed-form expression:
\begin{align}
\begin{split}
&\gamma^{(k+1)} ={\mathcal H}^\circ(y; s)= \frac{\mathcal H(y; s)}{\|\mathcal H(y; s) \|_2}, \text{ with }\\
& \qquad y = \gamma^{(k)} - \frac{1}{\rho} \nabla l(\gamma^{(k)}) = \gamma^{(k)} - \frac{1}{\rho} Z^\top W(Z\gamma^{(k)} - t). \end{split}\label{gammaiters}
\end{align}

\begin{theorem}\label{th:gammaoptalg}
Let \(\rho \ge \|Z^\top W Z\|_2\), where \(\| \cdot \|_2\) denotes the matrix spectral norm. For any initial point \(\gamma^{(0)}\) satisfying \(\|\gamma^{(0)}\|_0 \le s\) and \(\|\gamma^{(0)}\|_2 = 1\), the sequence \(\{\gamma^{(k)}\}\) generated by \eqref{gammaiters} produces non-increasing (and thus convergent) function values:
\[
l(\gamma^{(k+1)}) \le l(\gamma^{(k)}) \quad \text{for all } k \ge 0.
\]
 Furthermore, if   \(\rho > \|Z^\top W Z\|_2\), $\| \gamma^{(k+1)})  - \gamma^{(k)}\|_2\rightarrow 0$ as $k\rightarrow \infty$.
\end{theorem}
\begin{proof}
First, simple algebra shows
\begin{align*}
g(\gamma, \gamma^-) - l(\gamma) &= \frac{\rho}{2} \| \gamma - \gamma^- \|_2^2 - (l(\gamma) - l(\gamma^-) - \langle \nabla l(\gamma^-), \gamma - \gamma^- \rangle )  \\&=  \frac{\rho}{2} \| \gamma - \gamma^- \|_2^2 -   \frac{1}{2} ( \gamma - \gamma^-)^\top H(\xi)(  \gamma - \gamma^-)\\
&=\frac{1}{2} ( \gamma - \gamma^-)^\top (\rho I - H(\xi))(  \gamma - \gamma^-),
\end{align*}
where we applied the mean-value theorem,    $H(\xi)$ denotes the Hessian matrix of $l$ at   $\xi$ which is between   $\gamma$ and $\gamma^-$. Thus under the choice of $\rho$, $l(\gamma^{(k+1)})  \le  g( \gamma^{(k+1)}, \gamma^{(k)}) $ for any $k\ge 0$.

By the optimality of $ \gamma^{(k+1)}$, $g( \gamma^{(k+1)}, \gamma^{(k)}) \le g( \gamma^{(k)}, \gamma^{(k)}) = l( \gamma^{(k)})$ and the first conclusion follows. Moreover, from the inequality:  $l( \gamma^{(k)}) - l( \gamma^{(k+1)}) \ge \frac{1}{2} ( \gamma^{(k+1)} - \gamma^{(k)})^\top (\rho I - H(\xi))( \gamma^{(k+1)} - \gamma^{(k)}) \ge \frac{\rho - \|Z^\top W Z\|_2}{2} \| \gamma^{(k+1)} - \gamma^{(k)}\|_2^2 $, we obtain the second result.
\end{proof}

A summary of our algorithmic procedure is outlined in Algorithm \ref{Alg:SpaIso}. Some practical implementation notes:
(i) The provided values $\nu_A$ have been   baseline adjusted as described in \eqref{nucenter} (i.e., a preprocessing   $\nu_A  \leftarrow \nu_A - \nu_\emptyset$ for all $A \in 2^F$ is assumed).
We take $C$ as  1e+4  if $\| \nu\|_\infty\le 10$.
(ii) For $A = \emptyset$ or $A = F$, although $w_{\text{\tiny SH}}(A) $ takes infinite weights in theory, practically one can  assign a  weight equal to a large multiplier (e.g., 10) times the largest non-infinite weight (cf. ~\eqref{shapleyweights}), which  is often numerically sufficient to  enforce $\hat{T}(\nu_F) \dot = \sum_{j=1}^p \hat{T}(\hat{\beta}_j) $. (iii)
 It is unnecessary to explicitly form the diagonal matrix $W$; only the diagonal weights are required. Likewise, the sparsity of matrix $Z$ can be utilized. Additionally, key quantities such as $Z^\top W$ and $Z^\top W t$ can be precomputed prior to the iterative updates to improve computational efficiency. (iv) In Step 9), we employ a self-implemented, \textit{stack}-based weighted PAVA for improved efficiency.
(v) The paired data $(\nu_i, \hat t_i)$ approximate $T$ and form the basis for visualizing  $\hat T(\cdot)$. A heuristic for $\hat T^{-1}(\cdot)$ involves interpolating the inverted pairs $(\hat{t}_i, \nu_i)$, after averaging the $\nu_i$ values for any duplicate $\hat{t}_i$ coordinates to ensure a well-defined mapping.
Alternatives include fitting a second weighted isotonic regression $G(\hat{\delta})$ to $\nu$, or enforcing strict monotonicity in Step 9 (e.g., with an $\epsilon$-margin constraint) for an  invertible $\hat{T}$.
Overall, Algorithm \ref{Alg:SpaIso} is straightforward to implement and scales very well in practice.

\begin{algorithm}[t!]
    \caption{Sparse Isotonic Shapley Regression (\textbf{SISR}) Algorithm}
    \label{Alg:SpaIso}
    \textbf{Input:} $\nu = [ \nu_A ]_{ A \in 2^F} \in \mathbb R^{2^p}$ (baseline-adjusted,  such that $\nu_\emptyset=0$), sparsity level \(s\), design matrix \(Z \in \mathbb{R}^{2^p \times p}\), diagonal weight  matrix \(W\) (cf.~\eqref{Wdef}), and an initial vector $t^{(0)}\in \mathbb{R}^{2^p}$ (e.g., $C \nu$ with a large $C$  if $\| \nu\|_\infty$ is small, to improve  precision, and $C=1$ otherwise).
    \begin{algorithmic}[1]
        \State Initialize \(t \leftarrow t^{(0)} \), \(\gamma \leftarrow 0\)
\State   \(\rho \leftarrow \|Z^\top W Z\|_2\)
        \Repeat
         \While{not converged}

 \State  \(\xi \leftarrow \mathcal H(\gamma - \frac{1}{\rho} Z^\top W(Z\gamma - t); s)\)

 \State   \(\gamma \leftarrow \frac{ \xi }{\|  \xi  \|_2}\)

\EndWhile
            \State \(\delta \leftarrow Z\gamma\)
            \State Fit isotonic regression ~\eqref{subopt-iso} with $\delta, W, Z$ to update \(t\)
        \Until{convergence}
        \State \Return $t, \gamma$
    \end{algorithmic}
\end{algorithm}

\section{Data-Driven Insights}\label{sec:experiments}
\subsection{Domain Adaptation}

To propose a convenient noisy data generation scheme, let's revisit the statistical model defined in Section \ref{sec:method}, where the expectation \(\mathbb{E}[T^*(\nu)]=Z \gamma^*\), with $T(\cdot)$  applied componentwise. Assume without loss of generality that \(Z\) is structured using a bit generation process, with each row corresponding to the binary representation of \(i-1\) (e.g., the second row is $[1, 0, \ldots, 0]$). For \(\gamma^* = c_{0} [2^0, 2^1, \ldots,2^{p-2}, 2^{p-1}]^\top\), this yields
\[
\mathbb{E}[T^*(\nu)] = c_{0} [0, 1, \ldots, 2^p-2, 2^p-1]^\top
\]
where \(c_{0} = \sqrt{\frac{3}{4^p - 1}}\) ensures that \(\gamma^*\) is normalized. To simulate this in experiments, we generate \textit{noisy} versions \(\nu_A\) for all subsets \(A\) using
\[
\nu = Q(c_1 \cdot \sigma(U)) \in \mathbb{R}^{2^p}
\]
where \(U \in \mathbb{R}^{2^p}\) contains entries uniformly distributed between 0 and \(c_0(2^p - 1)\), approximately \(\sqrt{3}\) when \(p\) is sufficiently large. Here, \(\sigma\) denotes the permutation that {sorts} the elements of \(U\) in ascending order. An accurate estimator, \(\hat{T}\) or \(\hat{t}\), should then closely approximate the inverse transformation
$$
T^* = Q^{-1}/c_1.
$$ The inclusion of \(c_1\) is to ensure flexibility. 

Figure~\ref{fig:unifmodelsimu} presents the results under 6 different functional forms for the true transformation $T^*$: \emph{square root} ($T^* = (\cdot)^{1/2}$), \emph{fifth root} ($T^* = (\cdot)^{1/5}$), \emph{exponential} ($T^* = \exp(\cdot) - 1$), \emph{logarithmic} ($T^* = \log(\cdot +  1)$), \emph{tangent} ($T^* = \tan(\cdot)/c_1, c_1=10$),  and   \emph{normal distribution} ($T^* = \Phi(\cdot+c_2)/c_1, Q(\cdot) = \Phi^{-1}(c_1\cdot ) - c_{2},  c_1=1/\sqrt 3,c_{2} = Q(c_1  \sigma_{\min})$). As pointed out by a reviewer, comparing the estimated $\hat{T}$ directly to the true $T^*$, rather than their inverses on the $\nu$ scale, is  statistically   sound, as this aligns with our  additive Gaussian assumption   \eqref{TShap}.  Encouragingly, across all cases, the estimated transformation  $\hat{T}(\nu)$ closely aligns with the ground-truth $T^*(\nu)$, providing strong empirical evidence for the effectiveness of SISR in accurately recovering the underlying transformation structure.

\begin{figure}[!h]
    \centering
    \begin{subfigure}[b]{0.4\textwidth}
        \centering
        \includegraphics[width=\textwidth]{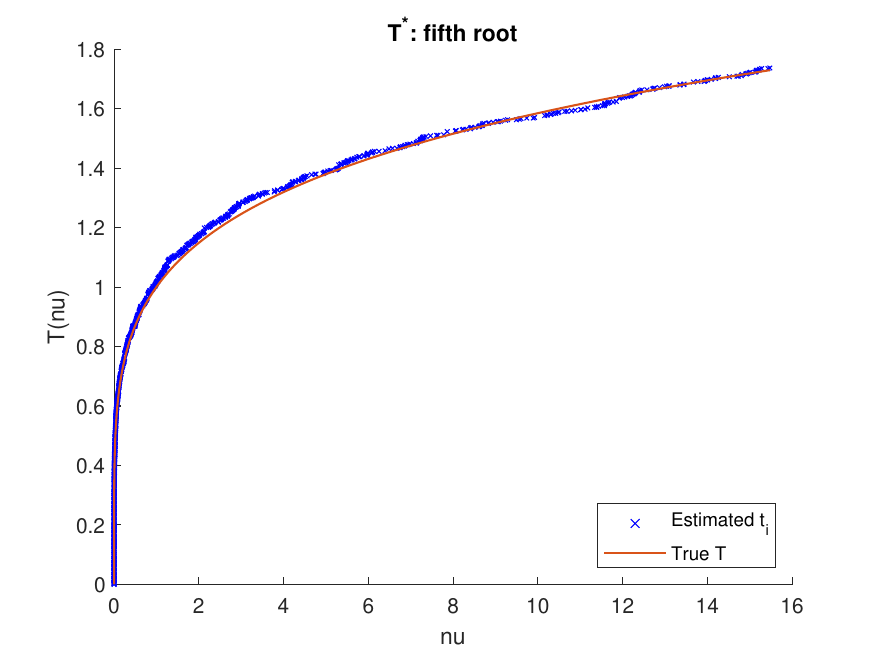}

    \end{subfigure}%
    \begin{subfigure}[b]{0.4\textwidth}
        \centering
        \includegraphics[width=\textwidth]{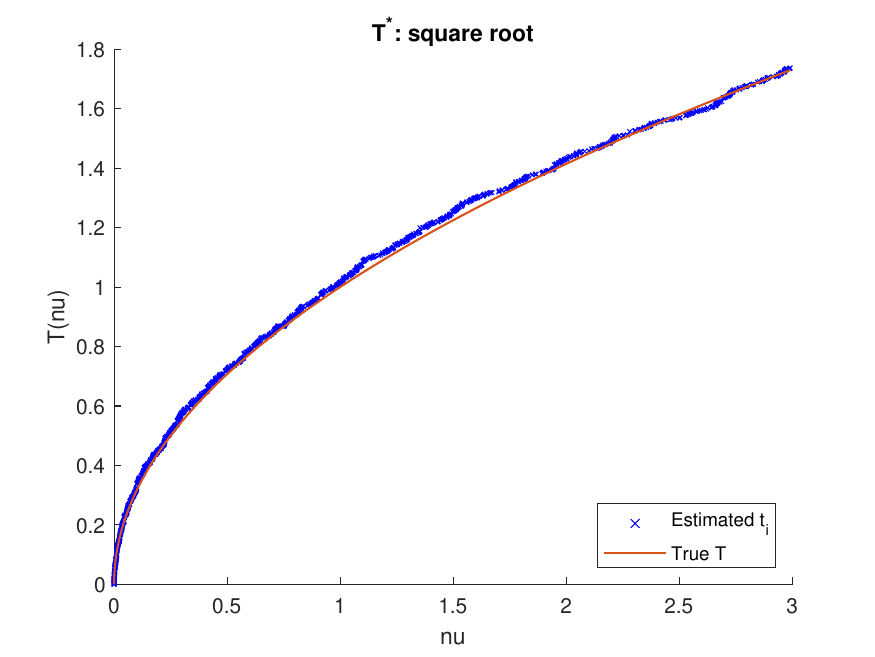}
    \end{subfigure}%
    \\
    \begin{subfigure}[b]{0.4\textwidth}
        \centering
        \includegraphics[width=\textwidth]{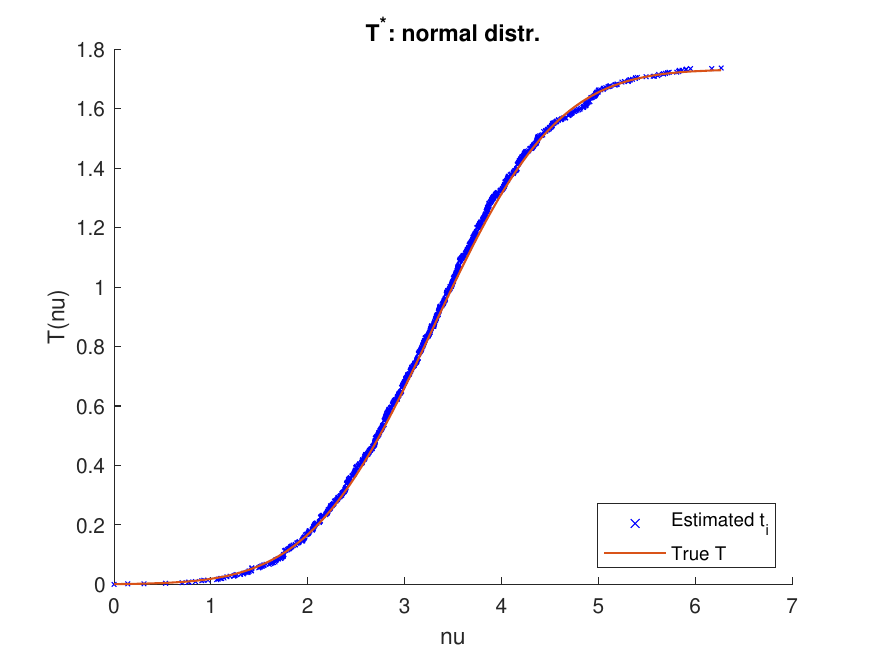}
    \end{subfigure}%
    \begin{subfigure}[b]{0.4\textwidth}
        \centering
        \includegraphics[width=\textwidth]{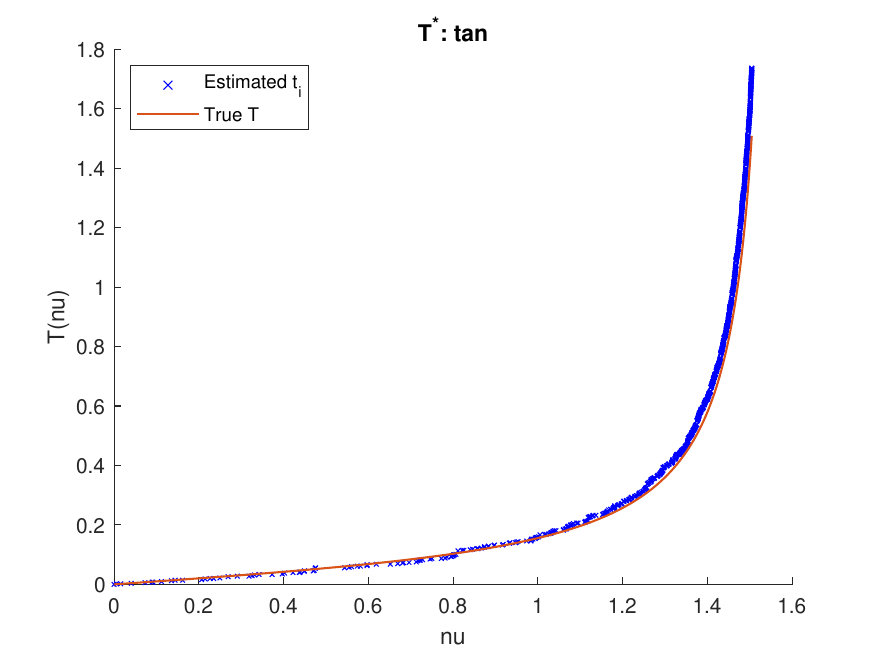}
    \end{subfigure}%
    \\
    \begin{subfigure}[b]{0.4\textwidth}
        \centering
        \includegraphics[width=\textwidth]{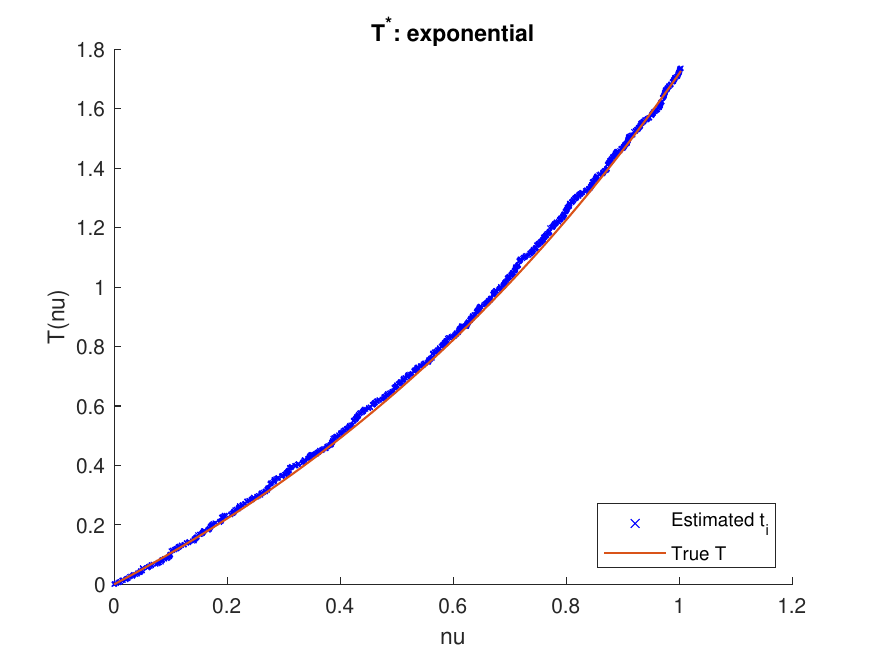}
    \end{subfigure}%
    \begin{subfigure}[b]{0.4\textwidth}
        \centering
        \includegraphics[width=\textwidth]{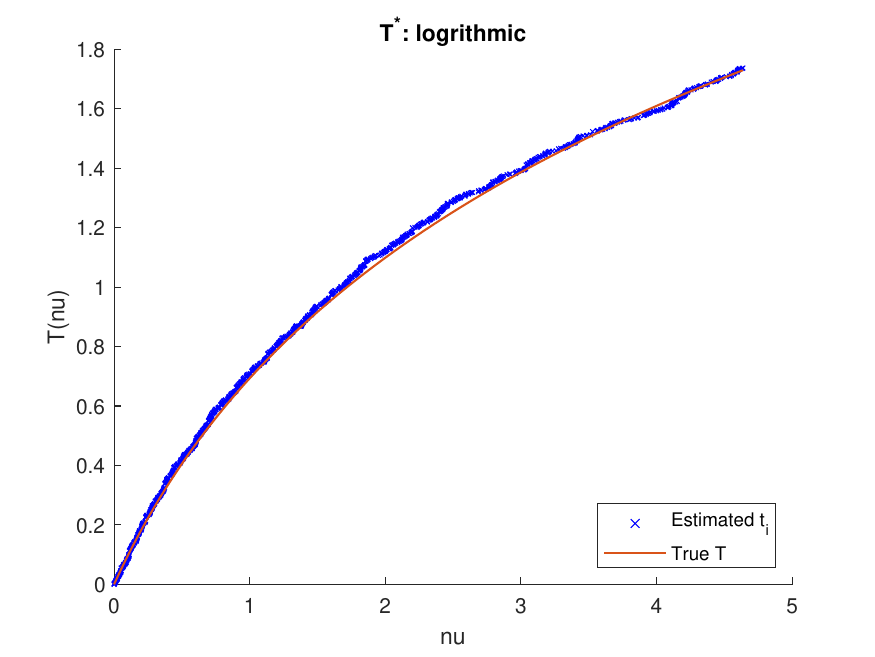}
    \end{subfigure}
    \caption{Estimated monotonic transformation $\hat{T}(\nu)$ (in \textit{blue}) versus the true transformation $T^*$ (in \textit{red}), for  $p=10$ under 6 different functional forms for $T^*$: the fifth root, square root,   normal distribution, tangent, exponential, and logarithmic transformations.     \label{fig:unifmodelsimu}}
\end{figure}

Finally, an additional experiment was conducted using data generated according to \(\nu_A = \max_{j\in A} \beta_j\), where \(\beta_j = j\). The resulting estimated transformation is displayed in Figure~\ref{fig:max}. The recovered transformation exhibits a pronounced increasing trend and  nonlinearity. The estimates are nearly perfectly correlated with the transformed ground truths, and the best-fit line passes through the origin up to a scaling factor.

\begin{figure}[!h]
    \centering
    \includegraphics[width=.48\textwidth]{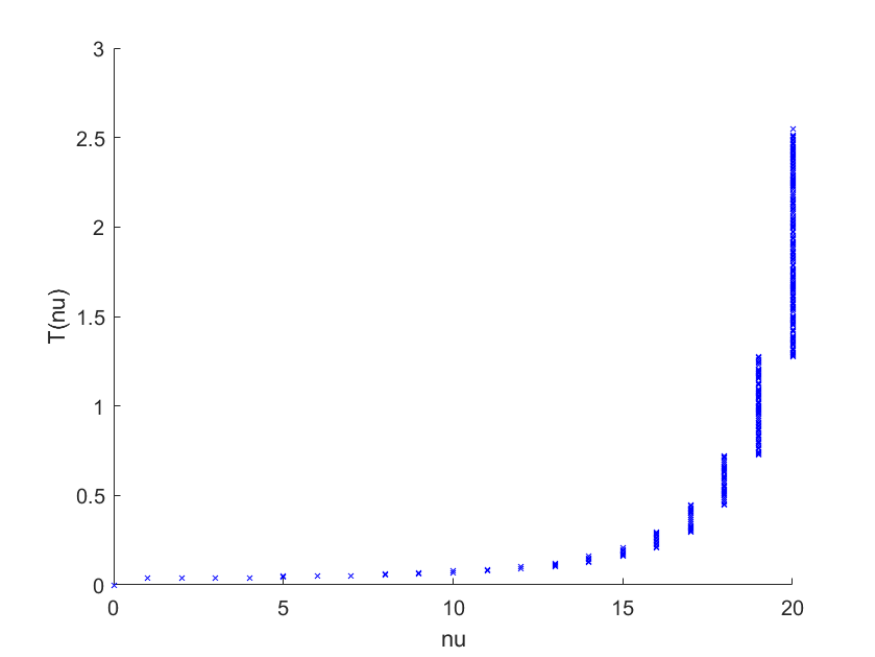}
 \includegraphics[width=.47\textwidth]{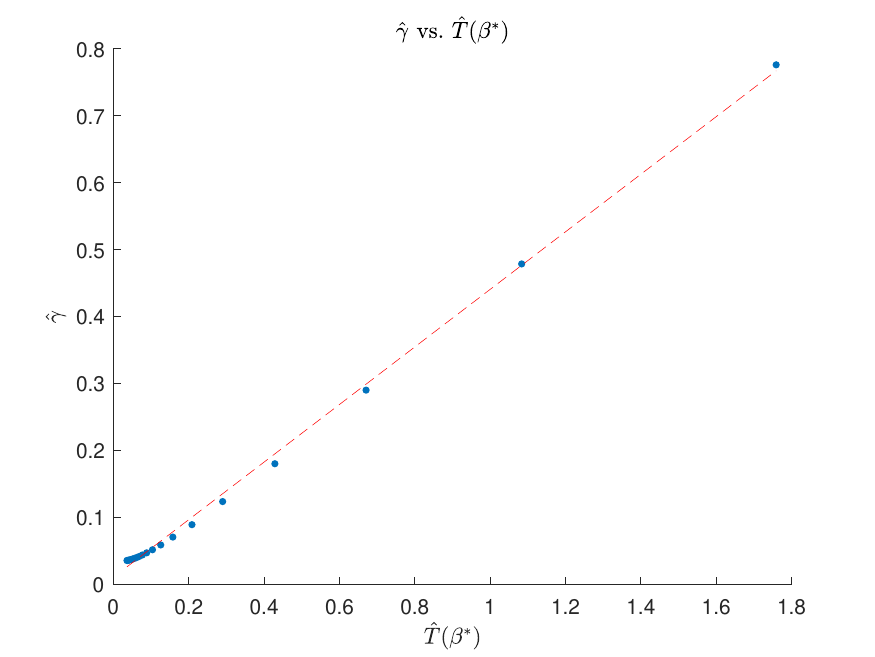}
    \caption{Estimated monotonic transformation \(\hat{T}(\nu)\) (left) and comparison  between  $\hat \gamma$ vs $\hat T(\beta^*)$ (right, showing an almost perfect correlation of 1.00) for \(p=20\) under a winner-takes-all setting. \label{fig:max}}
\end{figure}

\subsection{Sparsity Recovery}
We generate data according to the sparse-$\gamma$ model in Section \ref{sec:method},   with the true transformation set as the cubic root, $T^*(\cdot) = \sqrt[3]{\cdot}$. The true coefficient vector is $\gamma^* = [1/\sqrt{3},\ 1/\sqrt{3},\ 1/\sqrt{3},\ 0, \ldots, 0]^\top$,  a relatively weak signal with sparsity level $s^* = 3$. The noise variance is defined as $\sigma_A^2 = \sigma_0^2 / w_{\text{\tiny SH}}(A)$, with varying values of  $\sigma_0$. The sparsity level upper bound $s$ in running SISR is set to $ 1.5\, s^*$.
Performance is evaluated using two metrics. The first measures the alignment or affinity between the two unit-norm vectors: $\langle \hat{\gamma}, \gamma^* \rangle \times 100$ (denoted by \textbf{Affn}), serving as an index of estimation accuracy.  The second metric is the support recovery rate: $|\text{supp}(\hat{\gamma}) \cap \text{supp}(\gamma^*)| / s^* \times 100\%$ (denoted by \textbf{Supp}), reflecting the proportion of correctly identified nonzero components in the true support.
 Table \ref{tab:suppsimresults} reports results for varying values of $p$ and $\sigma_0$.  All results are averaged over 100 simulation runs.
\begin{table}[ht]
\centering
\caption{Affinity score (\textbf{Affn}) and support recovery rate (\textbf{Supp}) across different values of \( p \) and noise level \( \sigma_0 \). Larger values reflect better performance. }
\begin{tabular}{@{}cccccccc@{}}
\toprule
& \multicolumn{2}{c}{\(\sigma_0=1\text{e-3}\)} & \multicolumn{2}{c}{\(\sigma_0=5\text{e-3}\)} & \multicolumn{2}{c}{\(\sigma_0=1\text{e-2}\)} \\
\cmidrule(lr){2-3} \cmidrule(lr){4-5} \cmidrule(lr){6-7}
& Affn & Supp & Affn & Supp & Affn & Supp \\
\midrule
\( p=10 \) & 99.6 & 100\% & 99.6 & 100\% & 99.5 & 100\% \\
\( p=15 \) & 99.8 & 100\% & 99.9 & 100\% & 97.8 & 100\% \\
\( p=20 \) & 99.9 & 100\% & 87.9 & 100\% & 80.3 & 100\% \\
\( p=25 \) & 87.9 & 100\% & 74.0 & 100\% & 70.5 & 100\% \\
\hline

\hline
& \multicolumn{2}{c}{\(\sigma_0=5\text{e-2}\)} & \multicolumn{2}{c}{\(\sigma_0=1\text{e-1}\)} & \multicolumn{2}{c}{\(\sigma_0=2\text{e-1}\)} \\
\cmidrule(lr){2-3} \cmidrule(lr){4-5} \cmidrule(lr){6-7}
& Affn & Supp & Affn & Supp & Affn & Supp \\
\midrule
\( p=10 \) & 97.9 & 100\% & 88.7 & 98.7\% & 66.2 & 80.7\% \\
\( p=15 \) & 79.9 & 100\% & 70.9 & 98.0\% & 57.6 & 73.3\% \\
\( p=20 \) & 68.9 & 100\% & 63.2 & 96.0\% & 54.3 & 65.3\% \\
\( p=25 \) & 65.5 & 100\% & 60.6   & 90.7\% & 52.1 & 62.0\% \\
\bottomrule
\end{tabular}
\label{tab:suppsimresults}
\end{table}

As shown in the table, both performance metrics decline with increasing feature dimension \( p \) and noise level \( \sigma_0 \), as expected due to greater model complexity and reduced signal-to-noise ratio (SNR). Although not reported, running the algorithm without sparsity enforcement (i.e., \( s = p \)) yields noticeably worse affinity scores in high SNR settings (e.g., the first setting row of Table~\ref{tab:suppsimresults}). 
Compared to the affinity scores, the support recovery rate remains surprisingly strong even under challenging conditions, indicating that SISR consistently identifies the correct features.

We further investigated the impact of the sparsity level $s$ on computational time. In this experiment, we varied $s$ from 5 to 15, while fixing
$s^*=3$,  $p=15$,     $\sigma_0=5e\text{-}3$. As illustrated in the figure, lower sparsity levels generally lead to faster computation, highlighting the efficiency gains achievable when enforcing proper sparsity in the model.

\begin{figure}[!h]
    \centering
    \includegraphics[width=.75\textwidth]{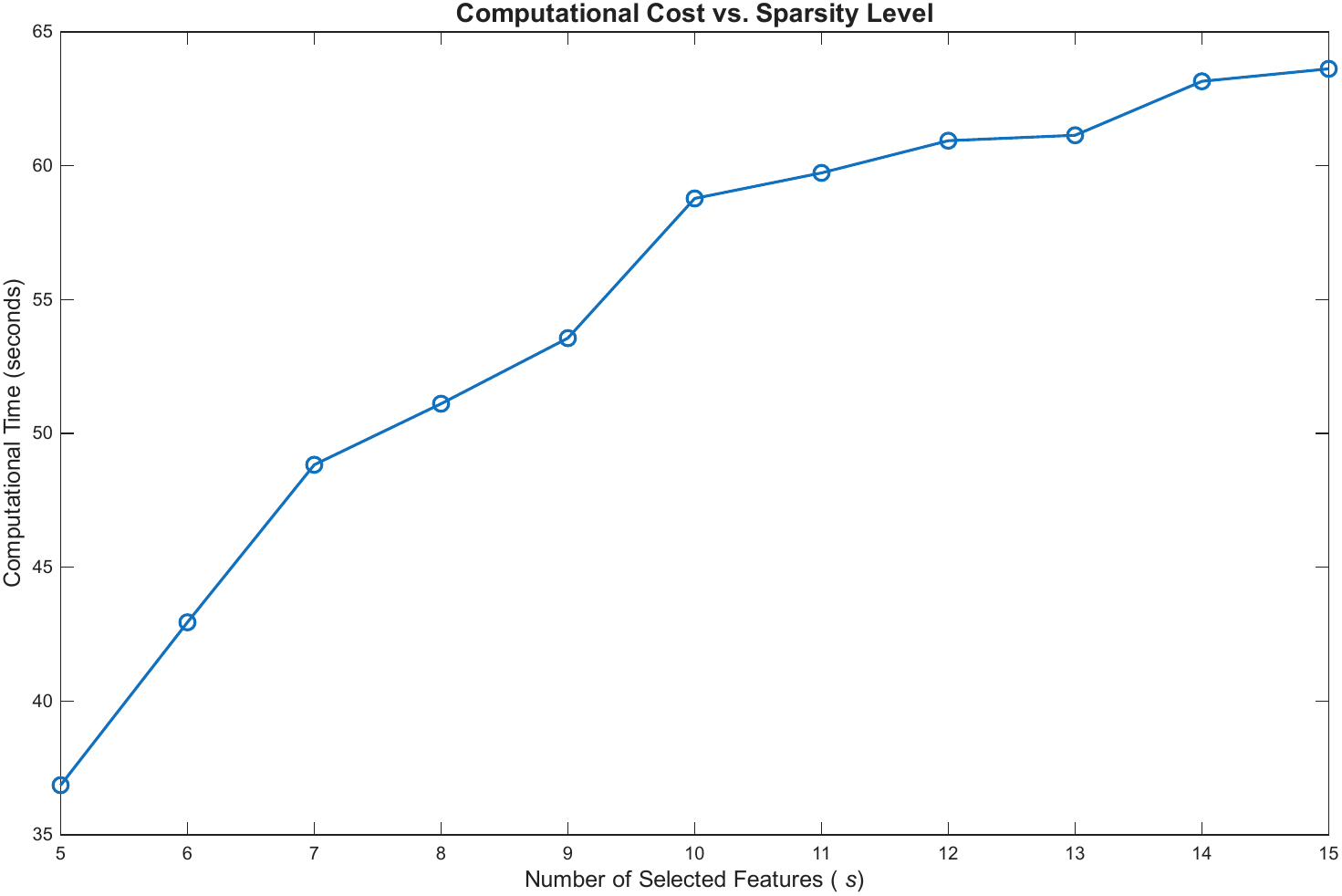}
    \caption{Computational time versus sparsity level. \label{fig:cost}}
\end{figure}

\subsection{$R^2$-Payoffs in Regression/Logistic Regression}\label{subsec:R2}
In regression settings, coalition values for feature subsets are commonly defined using the \emph{coefficient of determination} ($R^2$) obtained from retraining the model on each subset, reflecting the scaled improvement in model fit  \citep{Lipo2001,covert2021explaining}. Contrary to conventional expectations, our results unveil a novel insight: such a standard construction can fail to yield an inherently additive Shapley framework, especially when features are dependent or include irrelevant ones, which are almost certain to occur in practice.

To illustrate this phenomenon, let's consider the following simulation setup: \( y = X \alpha^* + \epsilon \), where \(\epsilon_i \sim \mathcal{N}(0, 1)\) and each row of \(X \in \mathbb{R}^{n \times p}\) is drawn from a multivariate normal distribution with mean zero and Toeplitz covariance \(\Sigma_{ij} = \theta^{|i-j|}\) with $\theta=0.5$. The true coefficient vector is set as \(\alpha^* = [3,\ 3,\ \ldots,\ 3]^\top\), and the sample size is \(n = 5p\). For each subset \(A\), we fit a regression model using the predictors indexed by \(A\) and define \(\nu_A\) as the resulting \(R^2\) value. Logistic regression is also considered in the classification setting,  $y_i \sim \mbox{Bernoulli}(\pi_i)$  and $\text{logit}(\pi) = X \alpha^*$, where we generate \(y\) according to a Bernoulli model and define \(\nu_A\) using the \textit{deviance}-based pseudo-\(R^2\).

\begin{figure}[!h]
    \centering
    \begin{subfigure}[b]{0.25\textwidth}
        \centering
        \includegraphics[width=\textwidth]{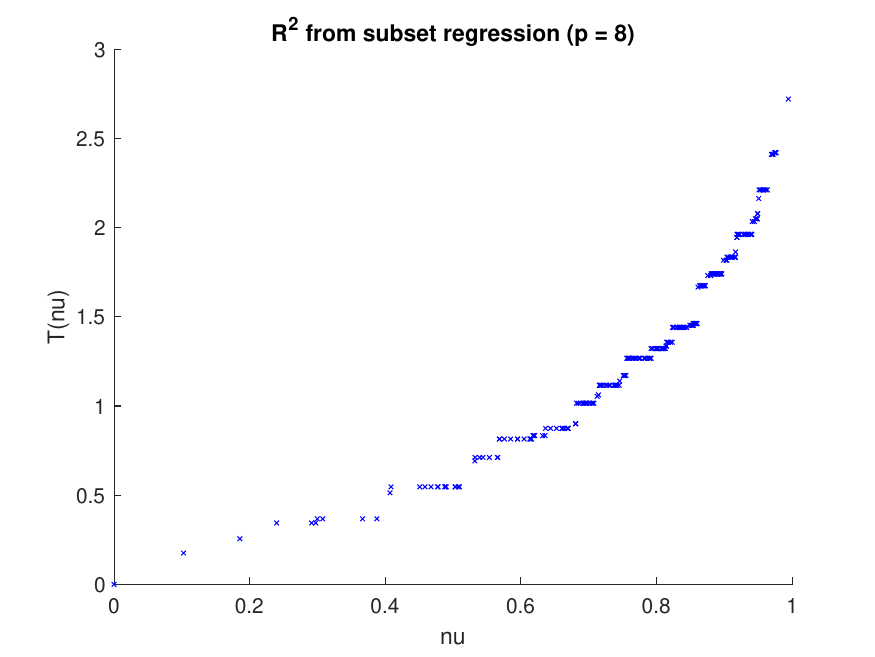}
    \end{subfigure}%
    \begin{subfigure}[b]{0.25\textwidth}
        \centering
        \includegraphics[width=\textwidth]{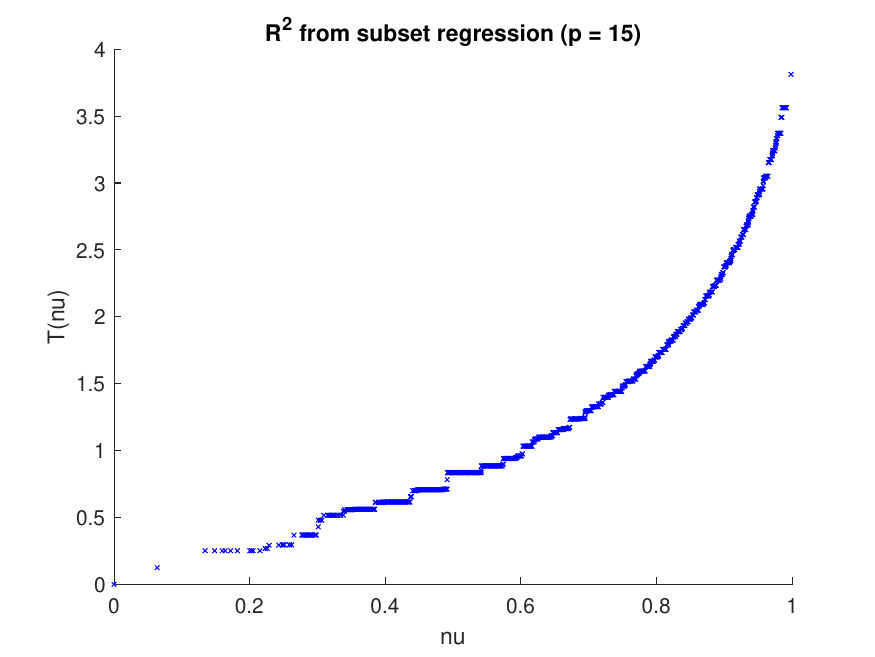}
    \end{subfigure}%
    \begin{subfigure}[b]{0.25\textwidth}
        \centering
        \includegraphics[width=\textwidth]{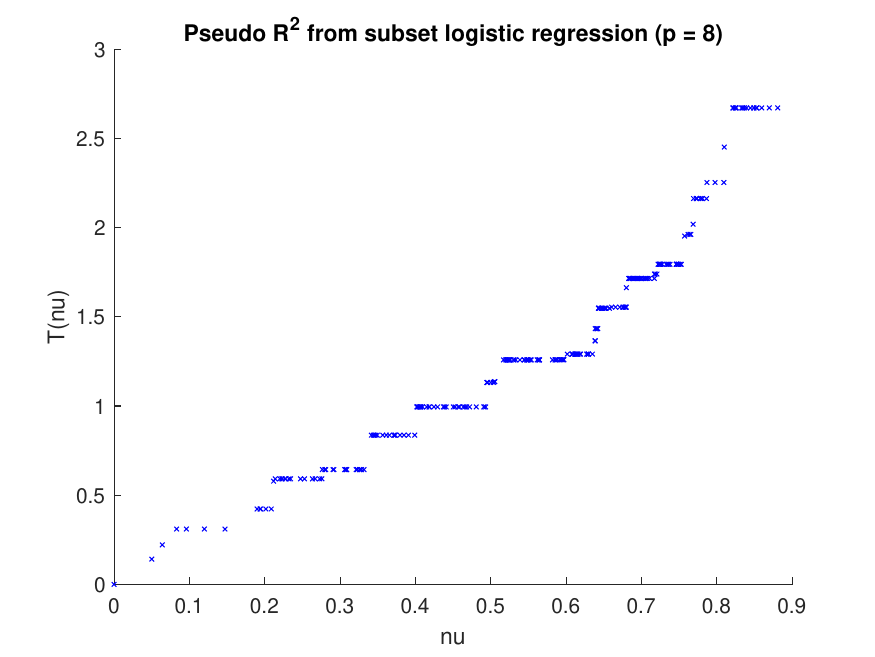}
    \end{subfigure}%
    \begin{subfigure}[b]{0.25\textwidth}
        \centering
        \includegraphics[width=\textwidth]{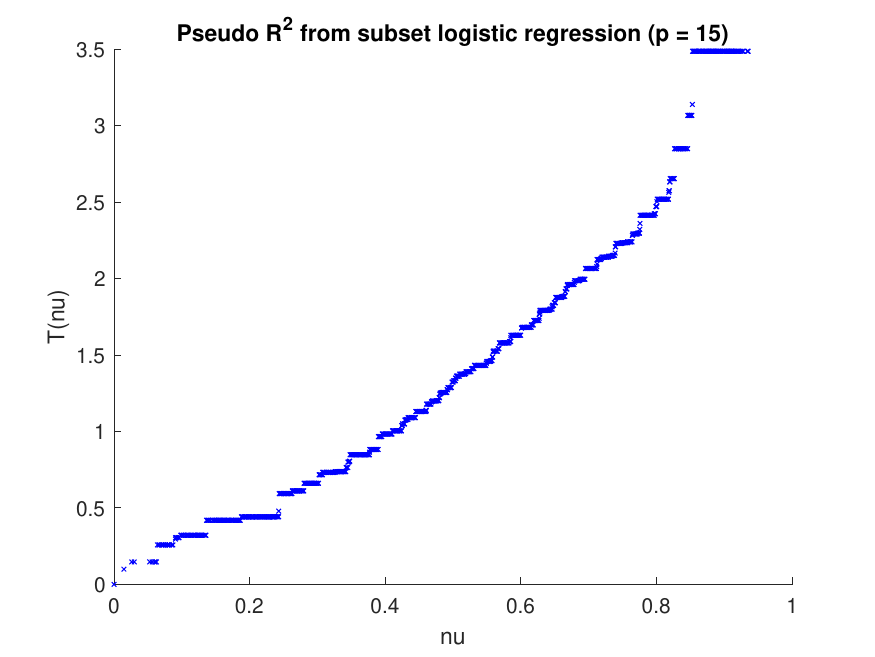}
    \end{subfigure}
    \caption{
     Estimated monotonic transformation \(\hat T(\nu)\) using regression-based \(R^2\)  and logistic regression-based pseudo-\(R^2\)  as the coalition worth function,  for \(p=8, 15\).
     \label{fig:regsimu}}
\end{figure}


As shown in Figure~\ref{fig:regsimu}, the estimated transformation \(\hat{T}(\nu)\) deviates significantly from linearity. For instance, in the regression case, even a simple logarithmic transformation fails to linearize the relationship, whereas a log-log transformation produces a nearly linear pattern---suggesting that the underlying transformation is super-exponential.\\

To examine whether model sparsity and feature correlation play a role in shaping the
transformation \(T\), we conducted a factorial simulation for linear
regression with \(p = 15\) and $s=p$.
We fixed the coefficient vector to  $
\alpha^{*}
   =  [3,0,3,0,\ldots,0 ]^{\mathsf T}$ with the sparsity level
$ s^{*}\in\{2,8,15\}$,
and varied the  correlation parameter
\(\theta \in \{0,\;0.5,\;0.9\}\).
Hence the design ranges from independent (\(\theta=0\))
to  collinear (\(\theta=0.9\)) predictors, and from very sparse
to fully dense signals.
Figure~\ref{fig:R2sparcorrsimu} displays the empirical
transformation~\(\hat T\) recovered in each setting.

\begin{figure}[!h]
    \centering
    \begin{subfigure}[b]{0.32\textwidth}
        \centering
        \includegraphics[width=\textwidth]{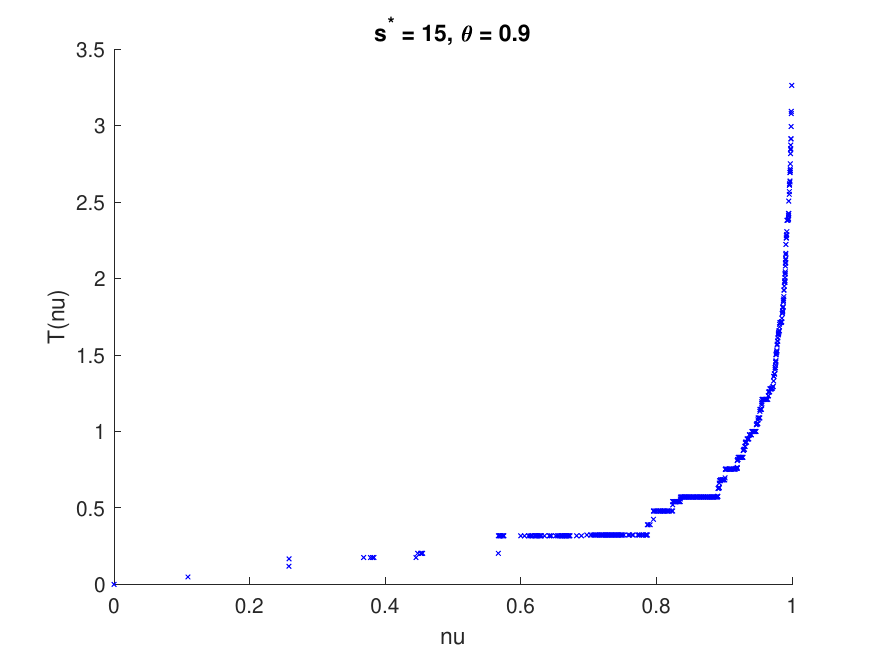}
    \end{subfigure}%
    \begin{subfigure}[b]{0.32\textwidth}
        \centering
        \includegraphics[width=\textwidth]{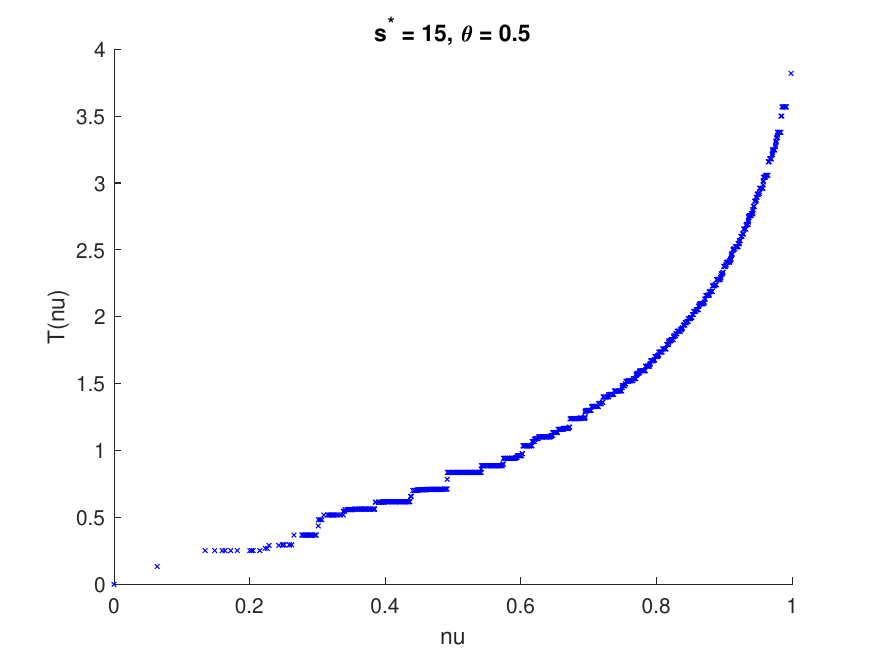}
    \end{subfigure}%
    \begin{subfigure}[b]{0.32\textwidth}
        \centering
        \includegraphics[width=\textwidth]{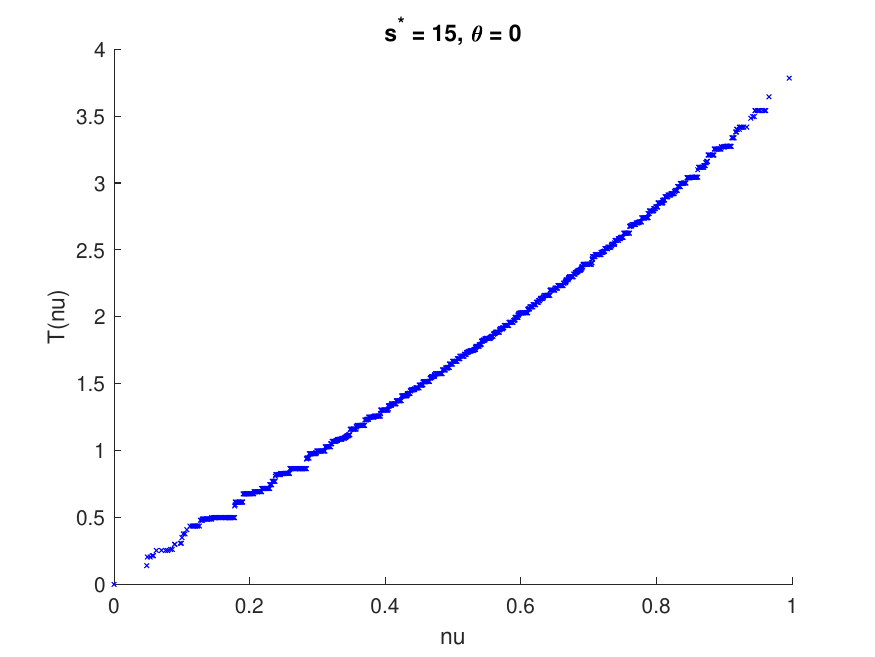}
    \end{subfigure}
\\
    \begin{subfigure}[b]{0.32\textwidth}
        \centering
        \includegraphics[width=\textwidth]{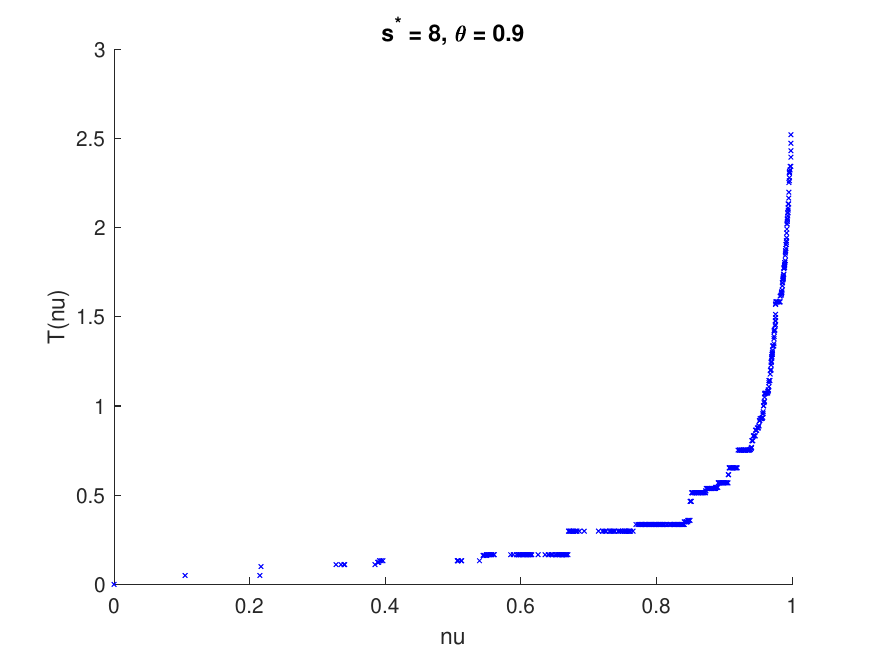}
    \end{subfigure}%
    \begin{subfigure}[b]{0.32\textwidth}
        \centering
        \includegraphics[width=\textwidth]{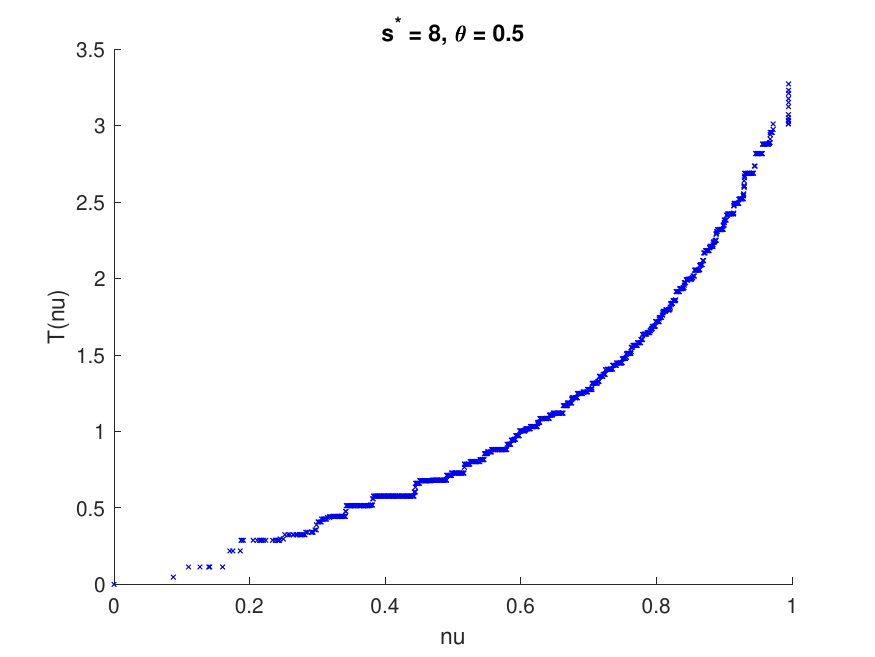}
    \end{subfigure}%
    \begin{subfigure}[b]{0.32\textwidth}
        \centering
        \includegraphics[width=\textwidth]{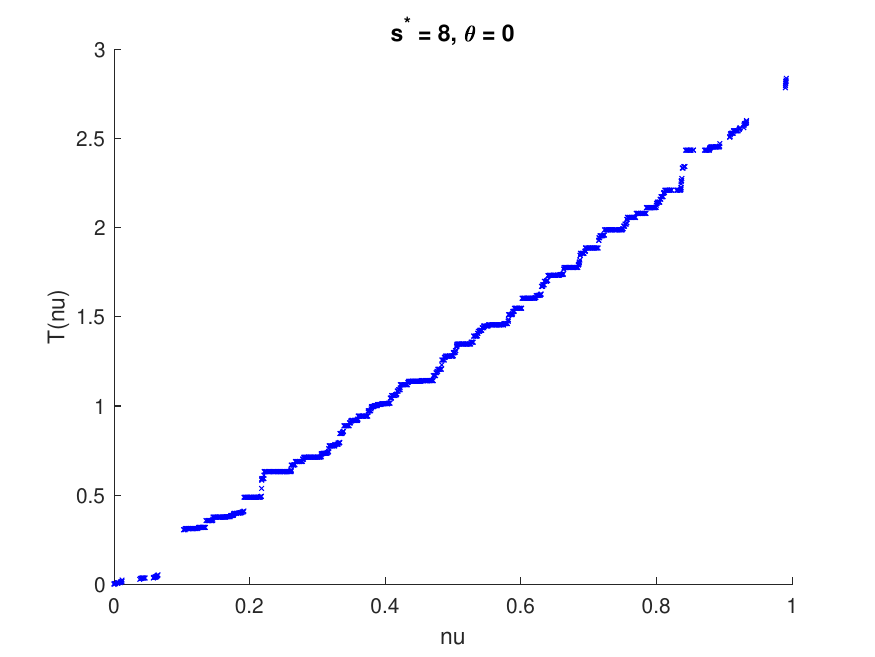}
    \end{subfigure}
\\
    \begin{subfigure}[b]{0.32\textwidth}
        \centering
        \includegraphics[width=\textwidth]{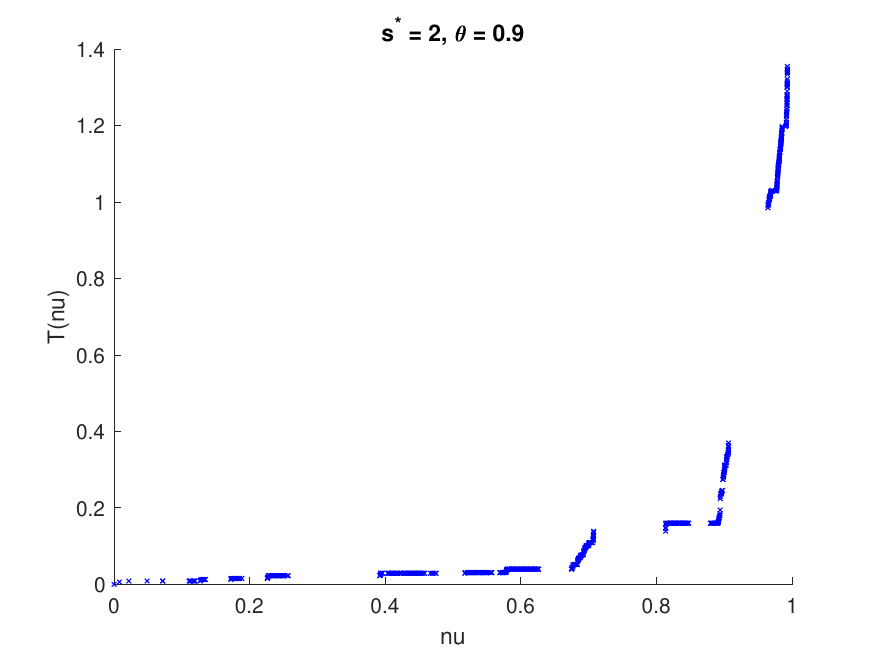}
    \end{subfigure}%
    \begin{subfigure}[b]{0.32\textwidth}
        \centering
        \includegraphics[width=\textwidth]{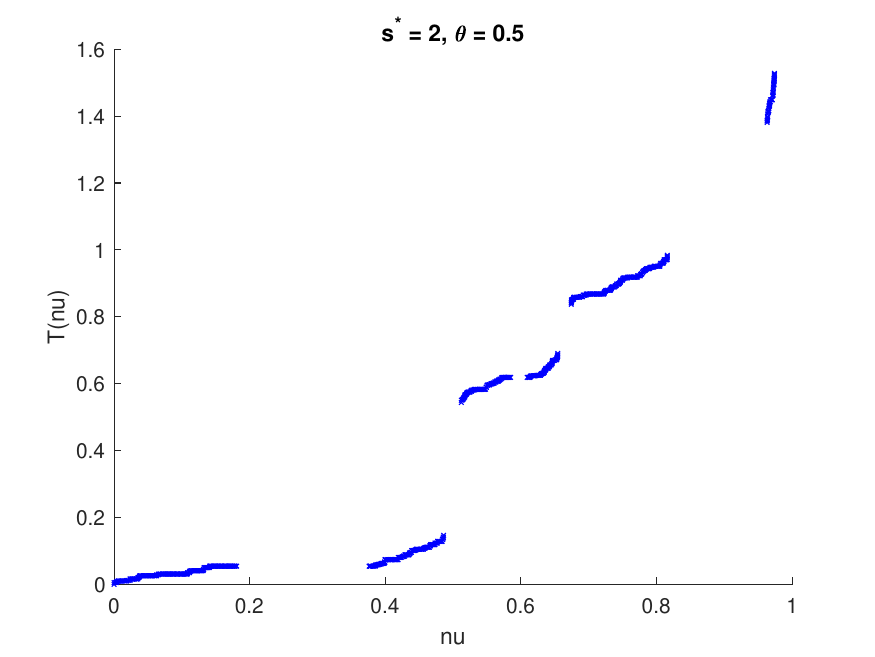}    \end{subfigure}%
    \begin{subfigure}[b]{0.32\textwidth}
        \centering
        \includegraphics[width=\textwidth]{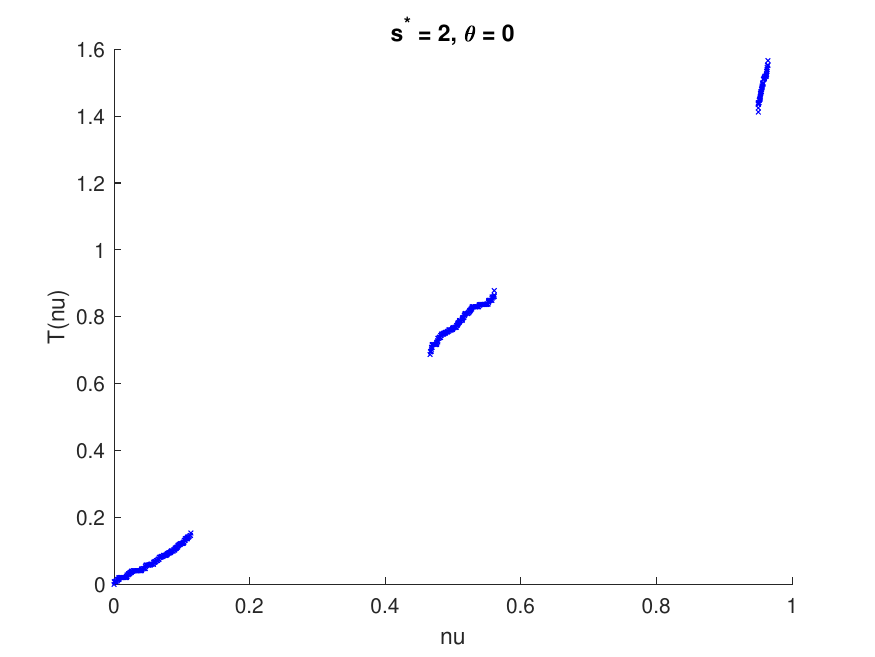}
    \end{subfigure}%
    \caption{Estimated monotonic transformation $\hat{T}(\nu)$ across varying sparsity levels ($s=15, 8, 2$, top to down) and feature correlation strengths ($\theta = 0.9, 0.5, 0$, left to right). The RIC criterion identifies $s=6$ as optimal.     \label{fig:R2sparcorrsimu}}
\end{figure}

Two main patterns emerge.
(i) \textbf{Correlation drives curvature:} as \(\theta\) increases,
\(\hat T\) deviates sharply from linearity, even in dense models.
(ii) \textbf{Sparsity introduces breaks:} at low \(s^{*}\) the curve becomes
piecewise, with segment slopes that differ substantially---another marker
of non-additivity.  Notably, pronounced nonlinearity appears even in the
independent, ultra-sparse case \((\theta = 0,\; s^{*} = 2)\), showing
that irrelevant features alone can distort raw worths.  Hence a \emph{monotone
nonlinear} transformation is indispensable when translating
\(R^{2}\)-based worth measures into the additive Shapley framework; without
it, either strong correlation or feature irrelevance breaks the required
additivity.
To the best of our knowledge, our results are the first to reveal that correlated features and the presence of irrelevant features can substantially undermine the additivity assumption in Shapley frameworks.

\subsection{Prostate Cancer}
We used the prostate data from \cite{Tibshirani1996} and took log(cancer volume) (\texttt{lcavol}) as the response, with clinical predictors including  log(prostate weight) (\texttt{lweight}), age (\texttt{age}), log(benign prostatic hyperplasia) (\texttt{lbph}), seminal vesicle invasion (\texttt{svi}, binary), log(capsular penetration) (\texttt{lcp}), Gleason score (\texttt{gleason}),   percentage Gleason 4/5 cells (\texttt{pgg45}), and log(prostate specific antigen) (\texttt{lpsa}). We computed SISR calibrated Shapley values  $\hat\gamma$ from $R^{2}$-worths at sparsity $s\in\{8,6,4\}$ and contrasted them with the conventional (raw) Shapley values in Fig.~\ref{fig:prostate}.

\begin{figure}[!h]
    \centering
    \begin{subfigure}[b]{.25\textwidth}
        \centering
        \includegraphics[width=\textwidth]{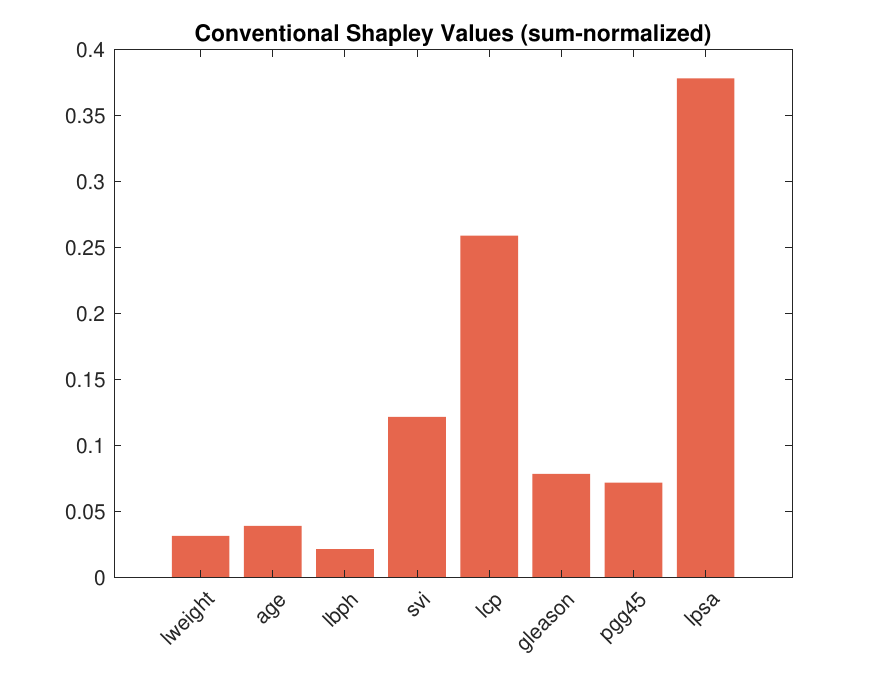}
    \end{subfigure}%
    \begin{subfigure}[b]{.25\textwidth}
        \centering
        \includegraphics[width=\textwidth]{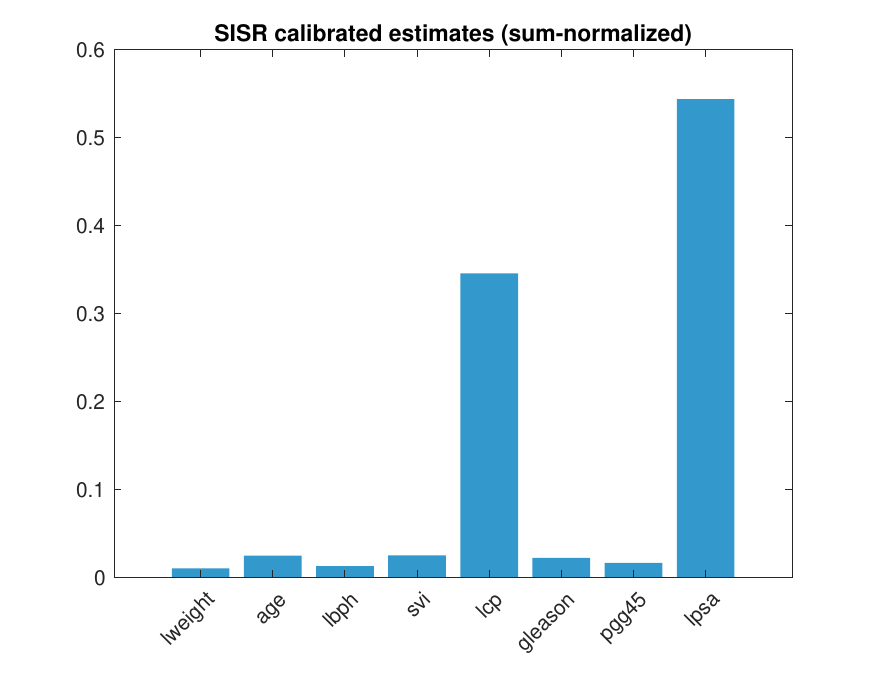}
    \end{subfigure}%
    \begin{subfigure}[b]{.25\textwidth}
        \centering
        \includegraphics[width=\textwidth]{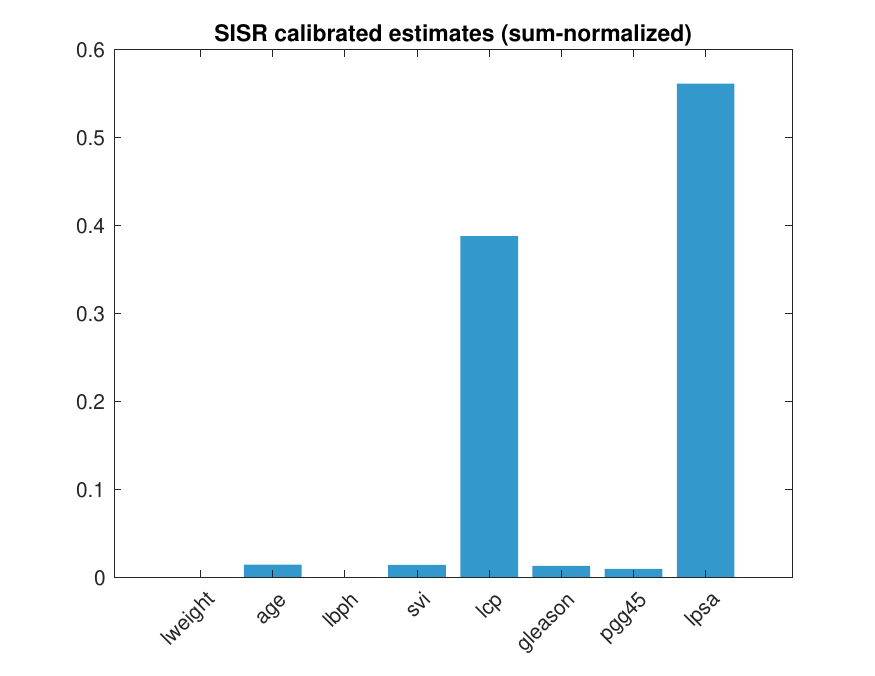}
    \end{subfigure}%
    \begin{subfigure}[b]{.25\textwidth}
        \centering
        \includegraphics[width=\textwidth]{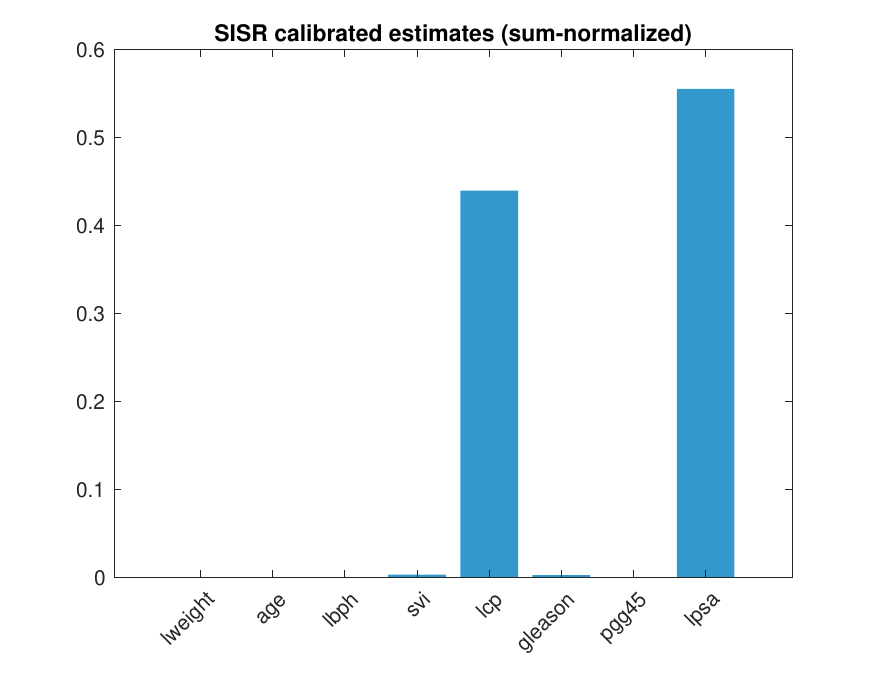}
    \end{subfigure}%
    \caption{Prostate data: conventional (raw) Shapley values (leftmost) and SISR-calibrated Shapley values at sparsity levels $s\in\{8,6,4\}$ (left to right). The RIC criterion identifies $s=6$ as optimal.    \label{fig:prostate}}
\end{figure}

Both schemes agree that \texttt{lcp} and \texttt{lpsa}   dominate the signal. The striking discrepancy concerns \texttt{svi}: the naive Shapley ranking elevates it to third place, claiming more than  10\%  of the total importance, whereas the calibrated $\hat\gamma$ assigns it virtually none.

This SISR finding turns out to be the one that aligns with established  evidence. Statistically, independent checks confirm that  \texttt{svi} contributes little: stepwise AIC and BIC both discard it, it is the final variable selected on the LASSO path \citep{Tibshirani1996}, and its $p$-value is as large as 0.6 in the full model. The conventional  Shapley result is also biologically implausible, since urological literature has consistently   shown that \texttt{svi} is not  a primary driver  \citep{Debras1998,Kristiansen2013}. Given this ground truth, the standard Shapley framework   fails to capture the true attributions.

This example cautions against applying \textit{off-the-shelf} Shapley formulas to raw coalition worths: without correction,  they can generate spurious importance values when the data includes correlated or irrelevant predictors. By jointly learning a monotone correction and imposing sparsity, our SISR procedure produces importance scores that align with corroborating diagnostics.

\subsection{Boston Housing}\label{subsec:boston}
The dataset \citep{harrison1978hedonic} captures socioeconomic and environmental factors for a suburban area of Boston, including  percentage of lower status of the population (\texttt{LSTAT}), weighted distances to five Boston employment centers (\texttt{DIS}), average number of rooms per dwelling (\texttt{RM}), bordering Charles River  (\texttt{CHAS}), and others,  alongside the response, the median value of owner-occupied homes. We trained an \textsc{XGBoost} regression model \citep{chen2016xgboost} using the hyper-parameter configuration given by \cite{maniar2023kaggle}.
After fitting the boosted tree ensemble, we evaluated the global performance of each feature subset using the interventional variant of SAGE    \citep{covert2020understanding}. For tree ensembles, SAGE adapts the fast path-traversal algorithm from TreeSHAP   \citep{Lundberg2020}  to estimate these global contributions efficiently.  Two payoff functions are considered: the negative mean squared error $\nu_A^{\text{\tiny MSE}} = -\operatorname{MSE}(A)$ and a robust form $\nu_A^{\text{rob}} =  \exp(-c\operatorname{MSE}(A))$ with $c =50/\max_{A} \operatorname{MSE}(A)$. The latter corresponds to a concave loss function, a standard approach in robust statistics  for modeling a  outlier-robust preference \citep{huber2009robust,hampel2011robust}, which down-weights the penalty from large, unpredictable errors.
   Figure \ref{fig:boston} displays the resulting Shapley attributions.

\begin{figure}[!h]
    \centering
    \begin{subfigure}[b]{.9\textwidth}
        \centering
        \includegraphics[width=\textwidth]{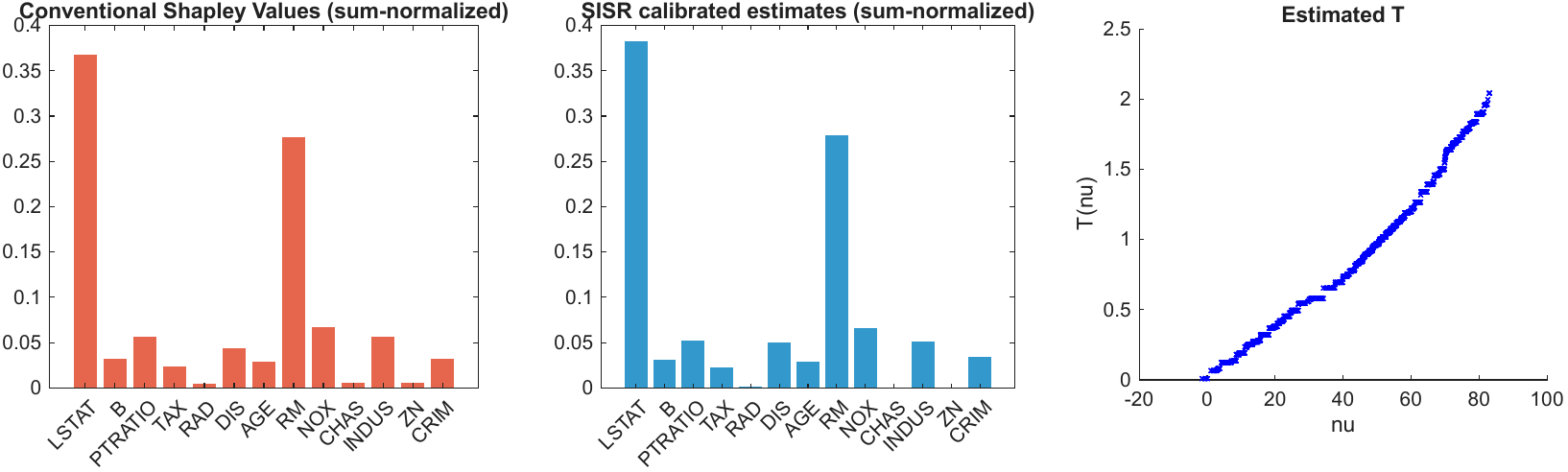}
        \caption{\footnotesize Negative-MSE payoff}
    \end{subfigure} \\ %
    \vspace{.15in}

    \begin{subfigure}[b]{.9\textwidth}
        \centering
        \includegraphics[width=\textwidth]{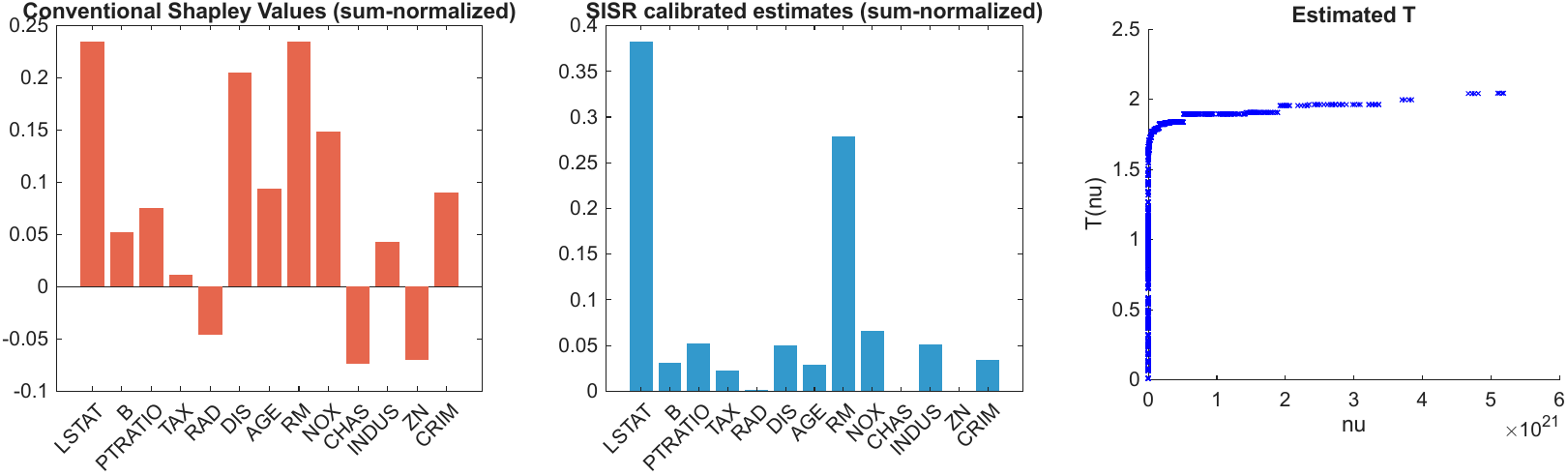}
        \caption{\footnotesize Robust payoff}
    \end{subfigure}%

    \caption{Boston housing: feature attributions
computed with conventional Shapley and SISR-calibrated  Shapley
 values  for  the negative-MSE payoff (top) and  the  robust payoff (bottom), along with the corresponding estimated monotone transformations.   \label{fig:boston}}
\end{figure}

The figure contrasts how the two attribution schemes respond to different payoff schemes.
Under the MSE payoff,   SISR has little to adjust---the scale is already compatible with linear additivity. In contrast, for the robust payoff, SISR produces a highly nonlinear transformation, 
 which  compensates for the distortions and preserves essentially the same attribution pattern observed under the MSE scale. The conventional Shapley values shift noticeably: the importance of \texttt{DIS} increases from minor to leading, and \texttt{CHAS} and several other variables even receive negative attributions. These sign and rank changes substantially alter the qualitative  interpretation of the game and reveal the standard procedure's sensitivity to the underlying payoff construction, whereas SISR remains robust.

\subsection{Bank Credit}
We analyze the South German Credit dataset, a benchmark for credit risk classification  \citep{covert2020understanding}. The dataset contains 1,000 observations, where the response is a binary indicator of credit risk, and $20$ predictor variables. These features include checking status, duration, credit history,
 age and so on. Following the experimental setup in \cite{covert2020understanding}, we trained a \textsc{CatBoost} classification model \citep{prokhorenkova2018catboost} on a training set.
With $p=20$ features, computing the full $2^{20}$ coalition values is computationally intractable. To approximate this full game, we selected a representative subset of 1,000 coalitions, using the efficient sampling strategy proposed by \cite{covert2021improving}. For each of these sampled coalitions, the payoff $\nu_A$ was then defined as the model's global performance  on a background test set, following the interventional SAGE methodology \citep{covert2020understanding}.
The two payoff functions we considered are the negative cross entropy, $\nu_A^{\text{ent}}$, and an exponential utility counterpart, $\nu_A^{\text{exp}} = -\exp(-c\nu_A^{\text{ent}})$. The first payoff corresponds to the standard    logistic log-likelihood, while the second payoff  models a strong risk-averse preference,   a foundational principle in modern economics and decision theory \citep{Pratt1964,Mas-Colell1995}.

\begin{figure}[!h]
    \centering
        \includegraphics[width=.31\textwidth,height=1.25in]{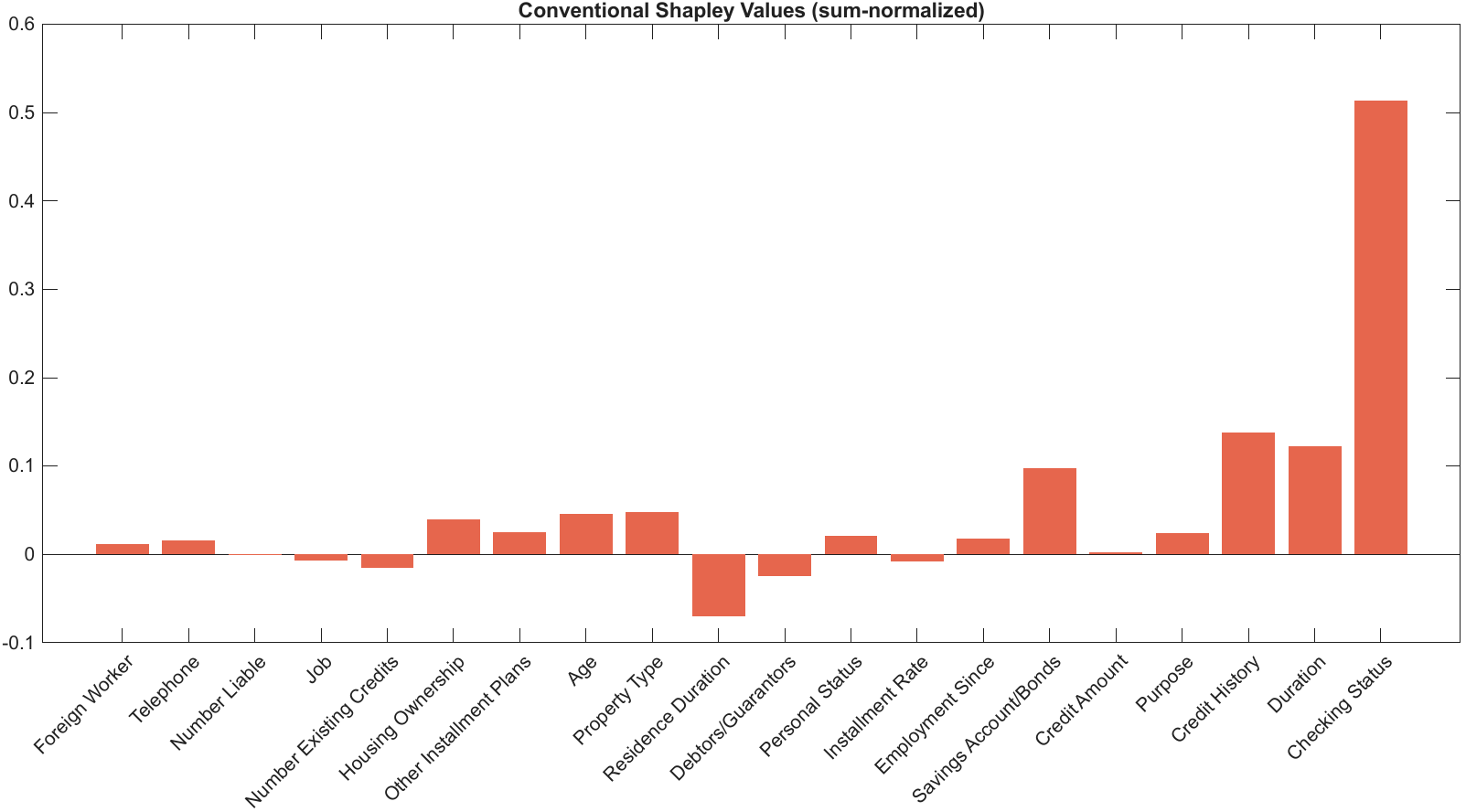}
        \includegraphics[width=.31\textwidth,height=1.25in]{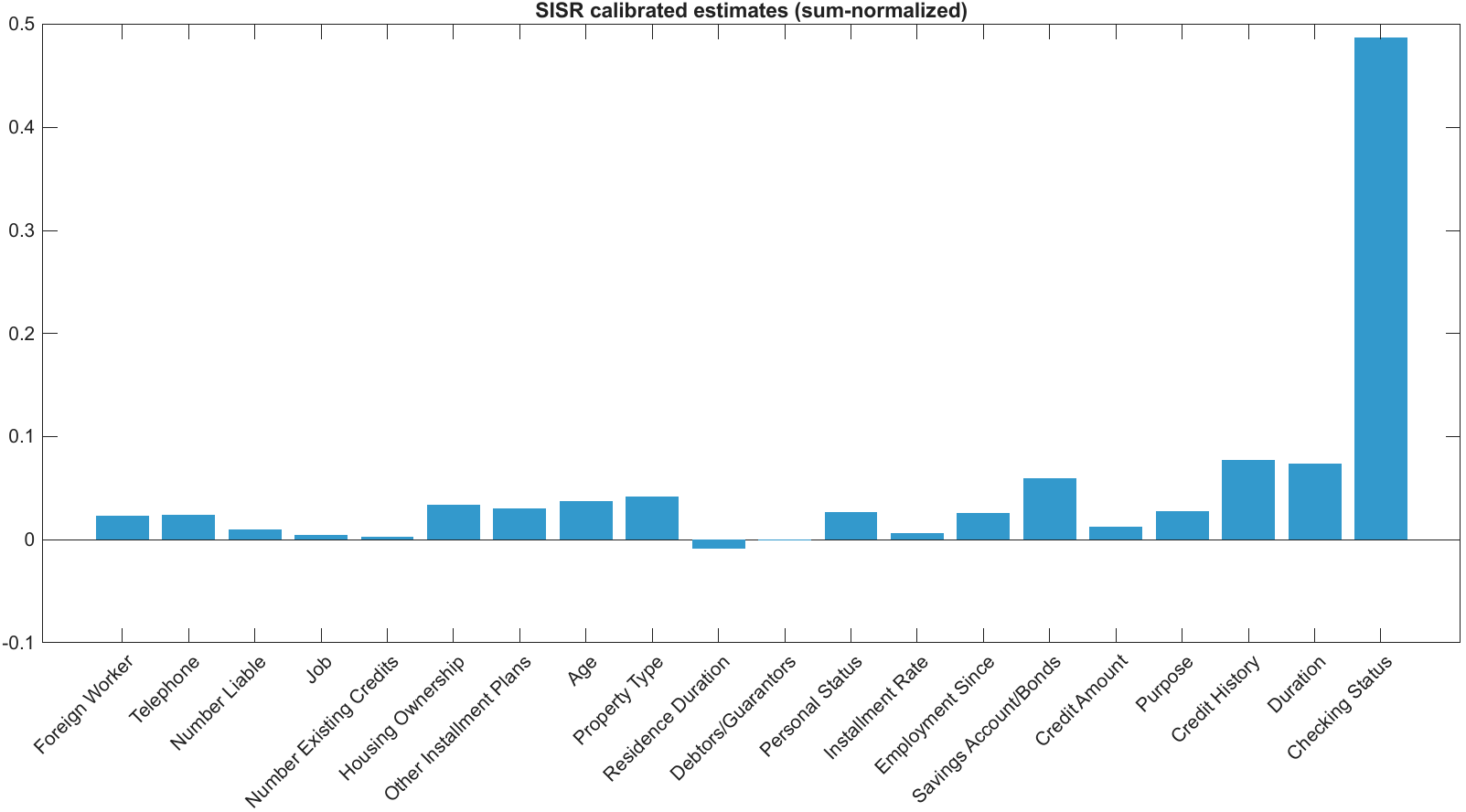}
        \includegraphics[width=.31\textwidth,height=1.25in]{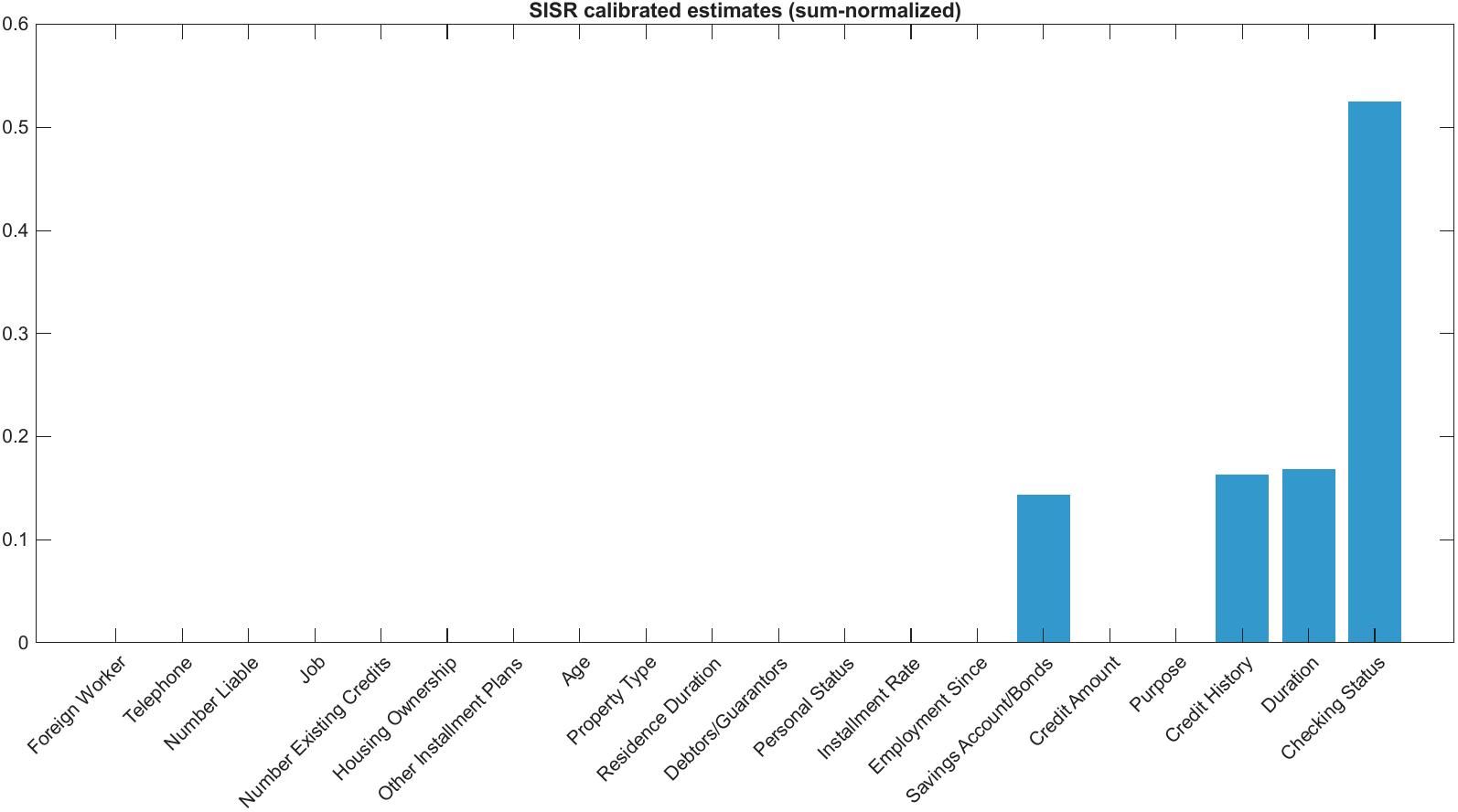}
\\
        \includegraphics[width=.31\textwidth,height=1.25in]{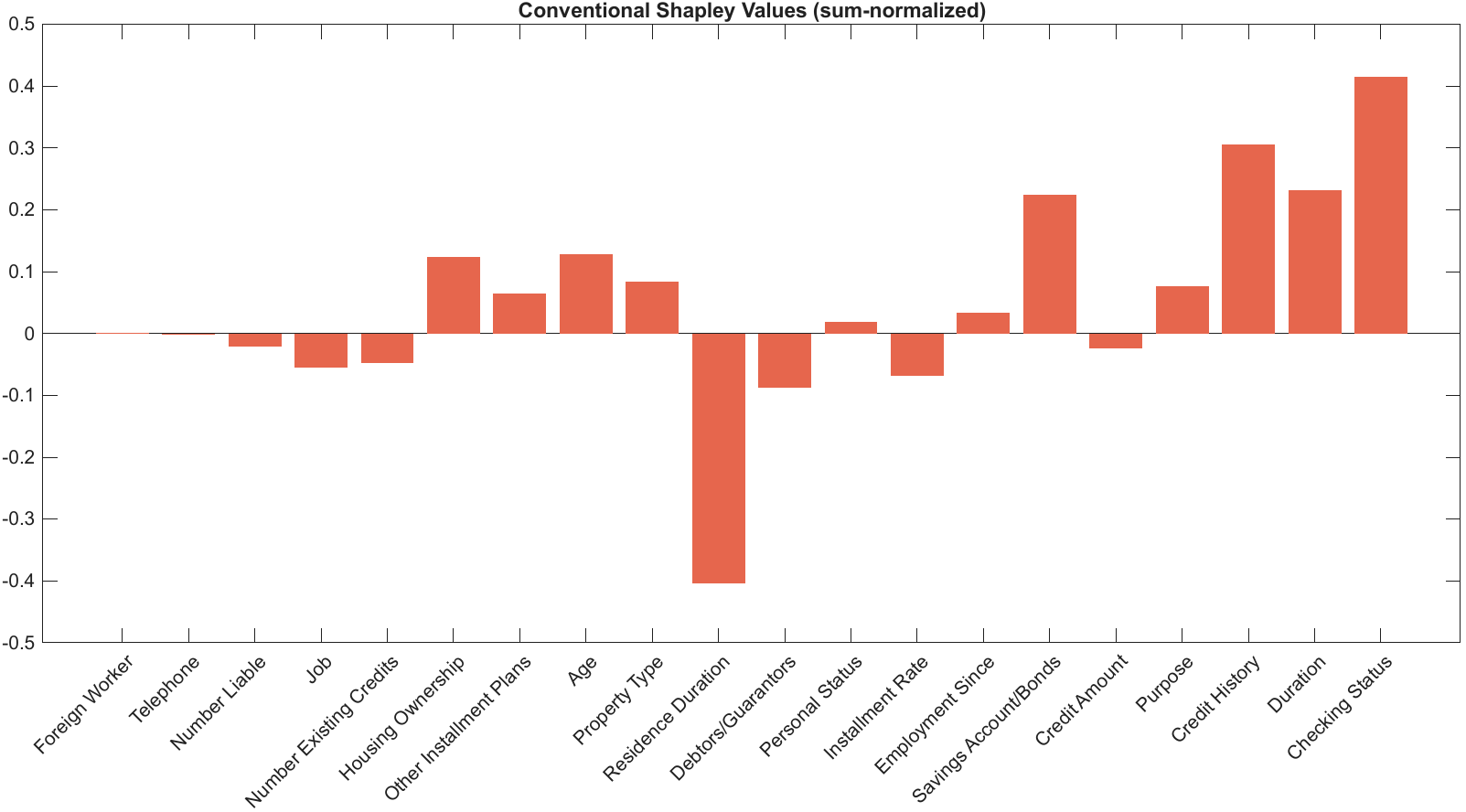}
        \includegraphics[width=.31\textwidth,height=1.25in]{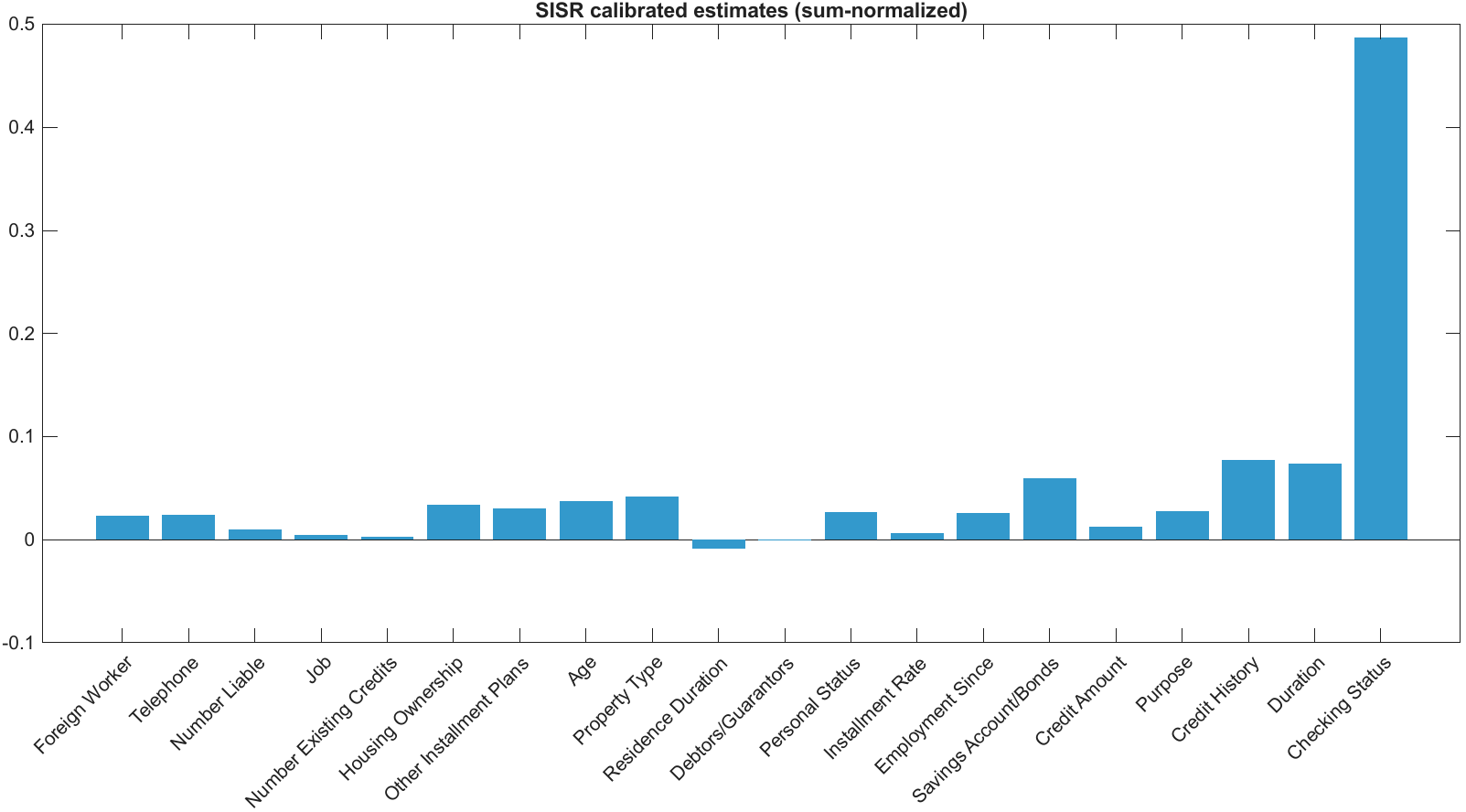}
        \includegraphics[width=.31\textwidth,height=1.25in]{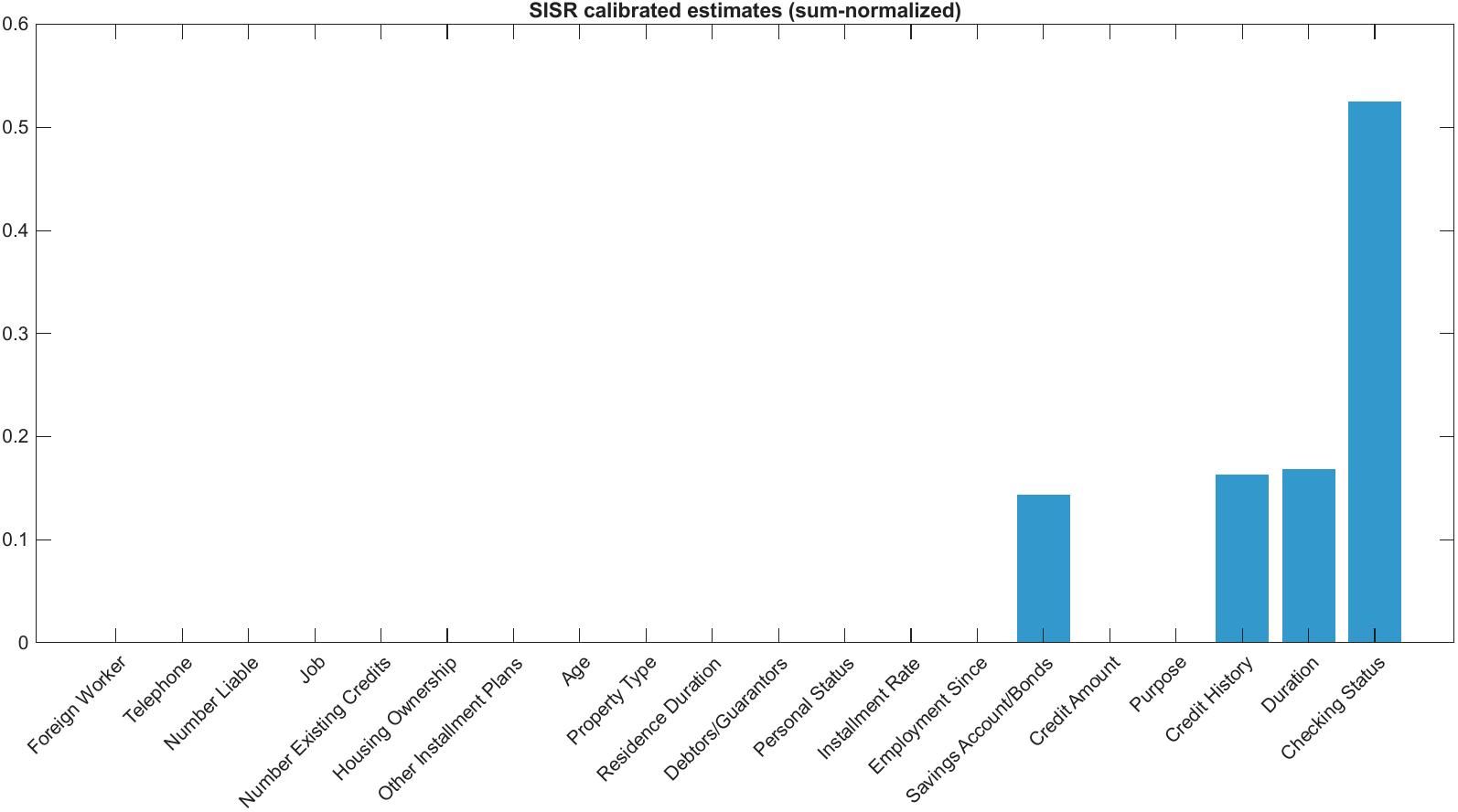}
        \caption{\footnotesize Bank credit data. The top row shows the feature attributions for the negative cross-entropy payoff, displaying (left to right): conventional Shapley values, SISR without sparsity, and SISR with $s=4$ (selected by RIC). The bottom row shows the same comparisons for the risk-averse exponential payoff. \label{fig:bankcredit}}
\end{figure}
Figure \ref{fig:bankcredit} displays the resulting feature attributions. The conventional Shapley results (top-left) differ  from those in the original SAGE analysis, likely due to sampling-based approximation errors. In particular,    \texttt{Residence Duration}   exhibits a spurious negative attribution. This instability becomes even more pronounced under the risk-averse exponential payoff (bottom-left), where the negative value grows to nearly four times its original magnitude.

In contrast, the SISR-calibrated attributions (with and without sparsity) remain remarkably stable.
\texttt{Residence Duration}    is assigned a near-zero contribution---consistent with the mild effect observed in \cite{covert2020understanding}. By removing nonlinear distortions in the payoff construction, SISR reveals the same sparse and interpretable structure across both settings, effectively filtering out the  distortions that undermine conventional Shapley estimates.

\subsection{Diabetes}

The Pima Indians Diabetes dataset \citep{smith1988pima} was collected by the U.S. National Institute of Diabetes and Digestive and Kidney Diseases. It records eight medical measurements for 768 women of Pima heritage near Phoenix, Arizona; the response indicates a physician's diagnosis of diabetes. Key predictors include plasma-glucose concentration two hours after an oral glucose-tolerance test (\texttt{Glucose}), body-mass index (\texttt{BMI}), and the diabetes-pedigree function (\texttt{DiabetesPedigreeFunction}).
We trained an \textsc{XGBoost} classifier \citep{chen2016xgboost} and tuned its hyper-parameters with \texttt{GridSearchCV} in \texttt{scikit-learn} \citep{pedregosa2011scikit}. Feature-subset worth was measured using two payoff functions: the negative cross-entropy (or logistic  log-likelihood) and the likelihood payoff (analogous to the utility variant in Section \ref{subsec:boston}). The results are shown in Figure \ref{fig:diab}.

\begin{figure}[!h]
    \centering
    \raisebox{5mm}{\includegraphics[width=.45\textwidth]{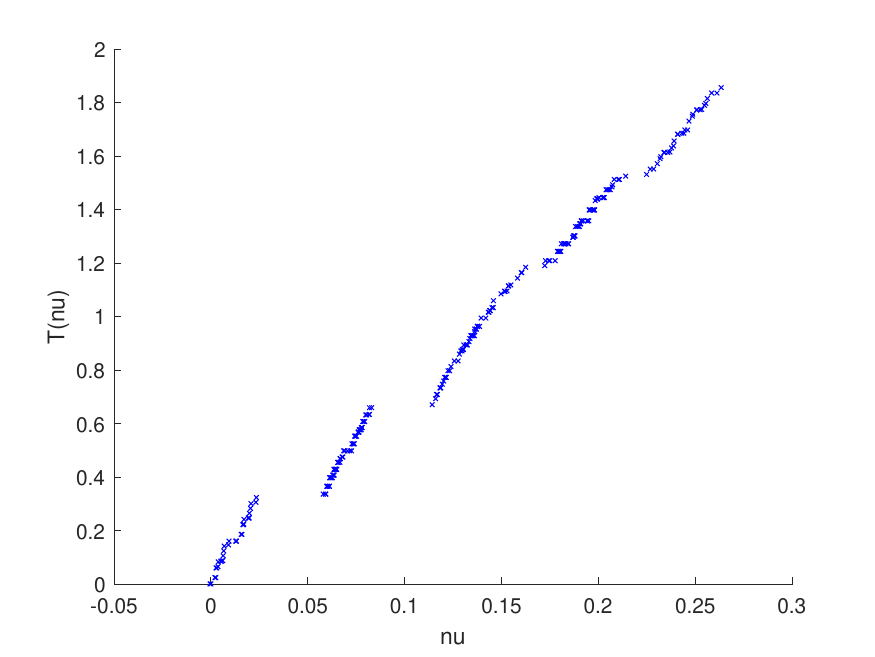}}
    \includegraphics[width=.45\textwidth]{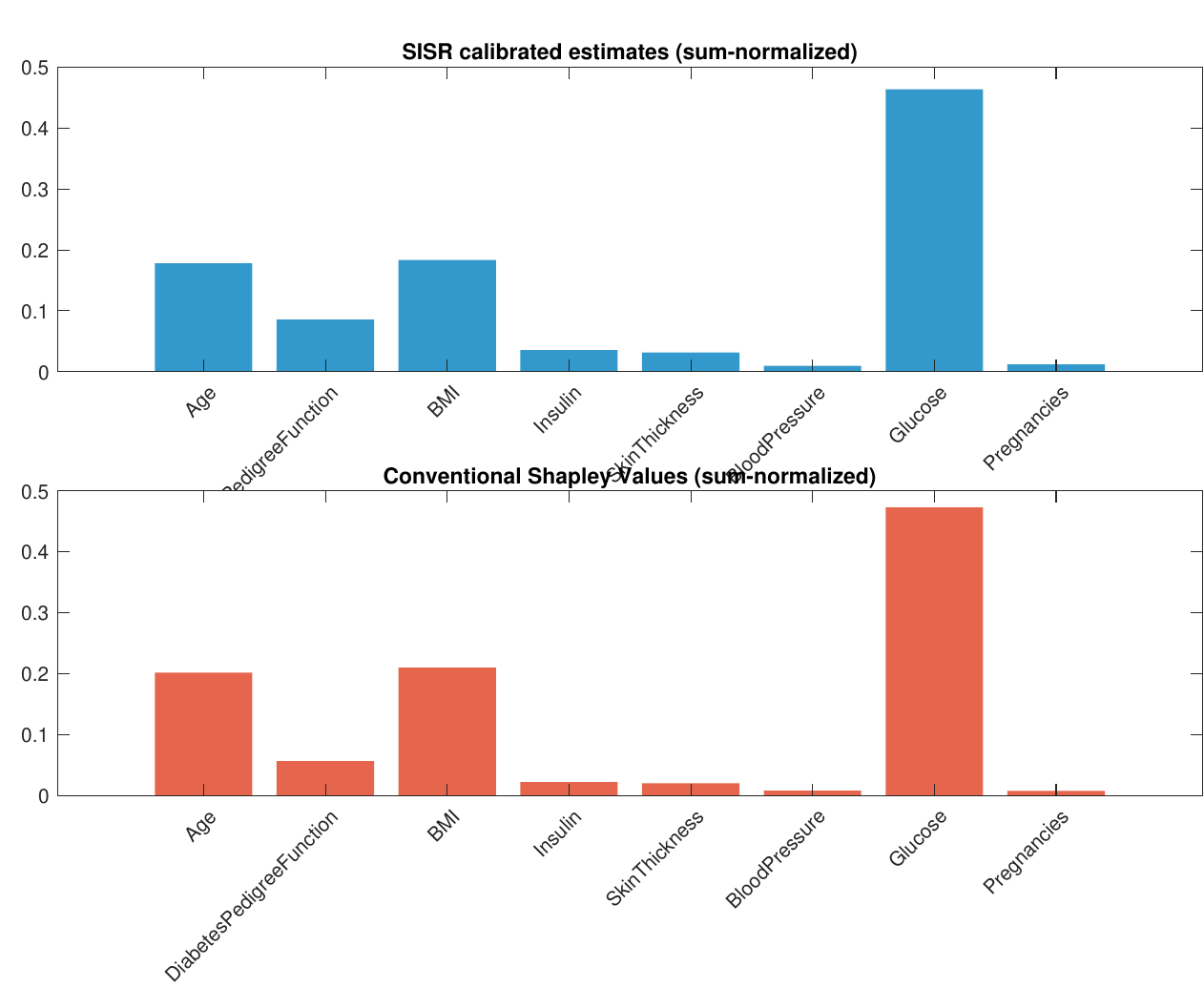}
\\
    \begin{subfigure}[b]{.25\textwidth}
        \centering
        \includegraphics[width=\textwidth]{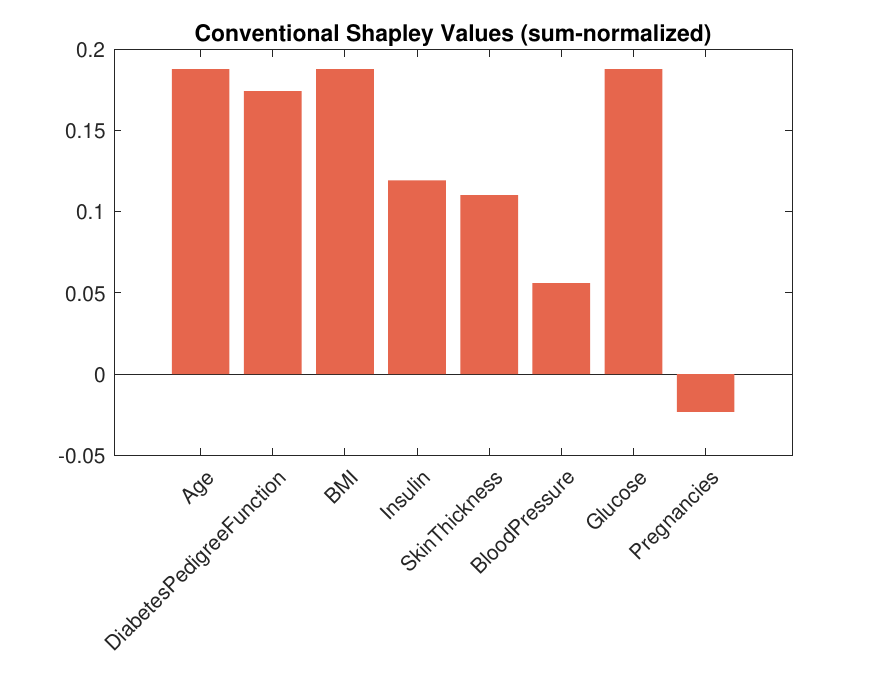}
    \end{subfigure}%
    \begin{subfigure}[b]{.25\textwidth}
        \centering
        \includegraphics[width=\textwidth]{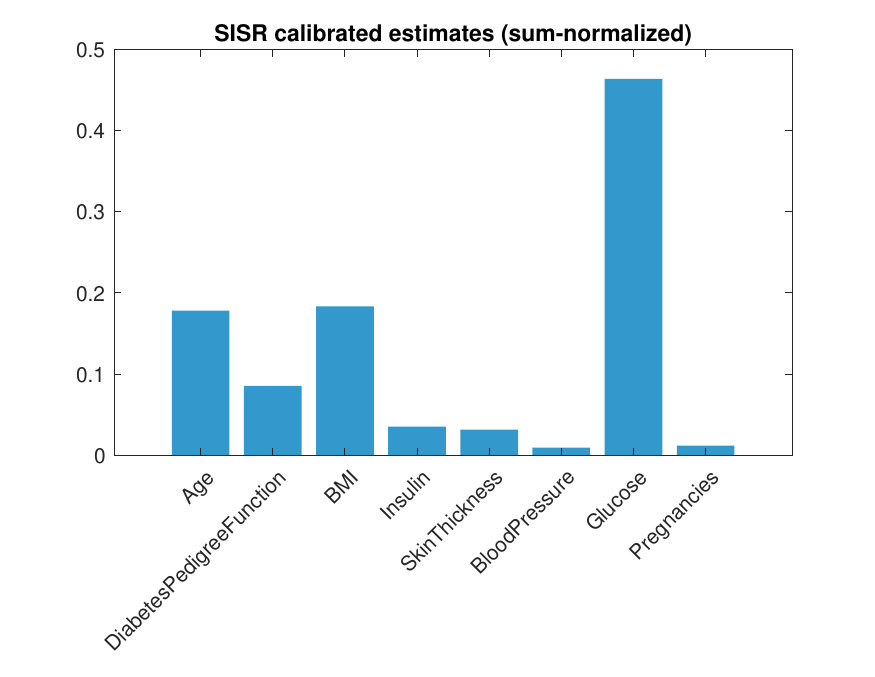}
    \end{subfigure}%
    \begin{subfigure}[b]{.25\textwidth}
        \centering
        \includegraphics[width=\textwidth]{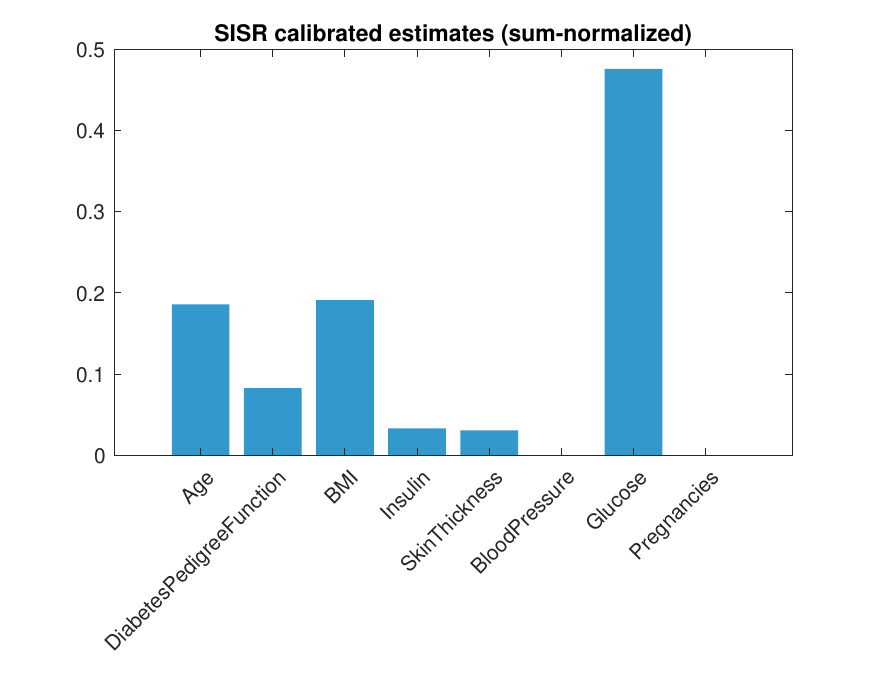}
    \end{subfigure}%
    \begin{subfigure}[b]{.25\textwidth}
        \centering
        \includegraphics[width=\textwidth]{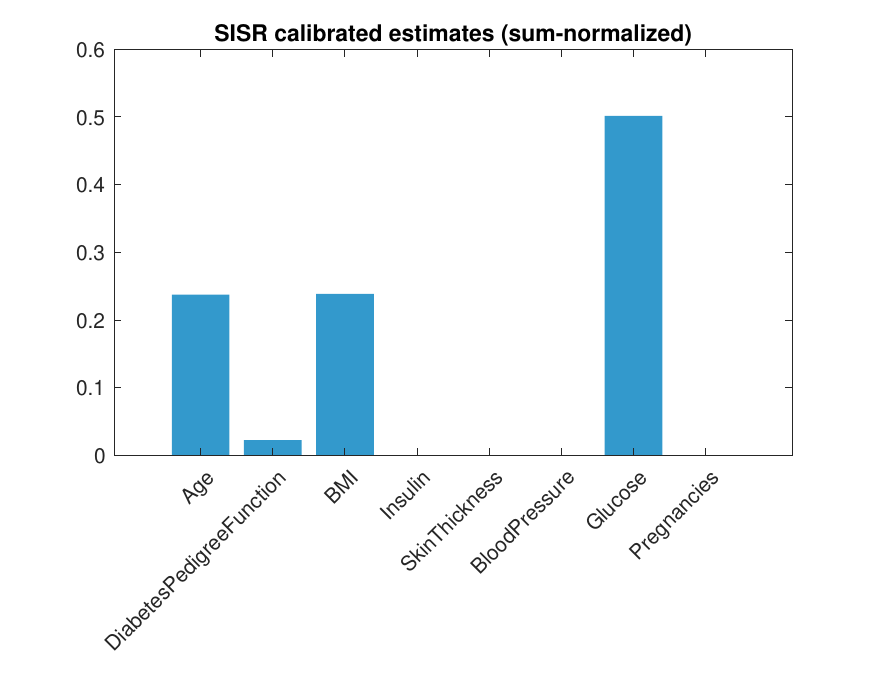}
    \end{subfigure}%
    \caption{Diabetes data. Top row (negative-entropy payoff):  estimated transformation $\hat{T}(\nu)$ obtained with SISR (left) and  feature attributions from SISR-calibrated versus conventional Shapley values (right).
Bottom row (likelihood or exponential payoff): feature attributions from the conventional method (leftmost) and from SISR for sparsity levels $s \in \{8,6,4\}$ (left to right). \label{fig:diab}}
\end{figure}

Using the negative-entropy payoff, the estimated transformation is piecewise linear---an indicator of sparsity in the underlying feature importances (cf. Section \ref{subsec:R2}). Although the segment slopes differ only modestly, the resulting shift alters several attributions: the largest discrepancy is for \allowbreak{\tt DiabetesPedigreeFunction}, whose importance is noticeably lower under the standard Shapley calculation than in the SISR-calibrated version.
 The bottom panel reports results for the likelihood payoff. Here, the conventional Shapley values are severely distorted: numerous features appear almost as influential as \texttt{Glucose}, and \texttt{Pregnancies} even becomes negative, illustrating how conventional attributions can be distorted by such a utility function.  By contrast, the SISR estimates without sparsity nearly replicate those obtained under negative entropy, and introducing sparsity leaves the ranking largely unchanged.
Overall, SISR delivers stable attributions across different value (payoff) functions, whereas the standard procedure is highly sensitive to which value function is chosen.
\section{Discussion and Extensions}\label{sec:ext}
Recent work has extended the classical Shapley value to account for higher-order interactions, including the Shapley-GAM framework and a wide family of Shapley Interaction (SI) indices, to provide a finer decomposition of the payoff function~$\nu$ \citep{GrabischRoubens1999,Marichal2000,Sundararajan2020,TsaiYehRavikumar2023,bordt2023}. Unlike the classical Shapley value, which is uniquely determined by its axioms, interaction indices are not: different extensions of the axiomatic framework or design criteria lead to distinct definitions and thus non-unique attributions, as exemplified by $k$-SII, STII,   FSII, among many others~\citep{muschalik2024}.
Computationally, all such methods scale unfavorably with the number of features~$p$, and even approximate variants remain expensive when $k$ grows beyond two \citep{Sundararajan2020}. From an interpretability perspective, explanatory power typically wanes with order: higher-order effects become increasingly hard to communicate and may overwhelm practitioners with ``information overload'' \citep{Baniecki2024}. From a statistical standpoint,   complex interaction terms may instead reflect noise or dependence-induced artifacts, yielding patterns that appear compelling but lack substantive validity
\citep{Hooker2007}.

In comparison, the SISR framework arises from a different---and often overlooked---misspecification. In many applications, the apparent non-additivity of the payoff does not reflect genuine high-order model logic but rather a nonlinear \emph{distortion} of the value function~$\nu$, introduced through preprocessing, surrogate construction, or sampling approximations. Existing interaction-based methods implicitly assume that~$\nu_A$ is a clean, faithful, and  \textbf{Gaussian} measure of model worth, so that any deviation from additivity must represent true interaction. Ignoring the noise, the $T^{-1}$-$\Sigma$-$T$ structure of SISR can also be rephrased (via Taylor expansion) in terms of  nonlinear  interactions---for example, if  $\beta_j\ge 0$ and $T(\cdot) = \log (1+ \cdot)$, the resulting value function $V_A(\{\beta_j\}_{j\in A})=\prod_{j\in A}(1+ \beta_j) -1= \sum_{j\in A} \beta_j + \sum_{j, j'\in A} \beta_j \beta_{j'} + \cdots$.  But  payoff constructions frequently violate Gaussianity: skewed, heavy-tailed, or bounded responses, as well as feature dependence or irrelevant covariates, can all distort $\nu$ and induce \textit{spurious} interaction where no genuine interaction would be found after a stabilization transformation.  SISR addresses this fundamental issue  by  {jointly} estimating a monotone transformation~$T$ and attribution effects, recovering an additive, Gaussian working model that restores interpretability.

We can unify stabilization and interaction to form a powerful framework for nonlinear XAI.    For example, we can generalize~\eqref{TShap} by  including interaction terms on the transformed scale:
\[
T(\nu_A) \sim \mathcal{N}\big( \sum_{j \in A} T(\beta_j^*) + \sum_{\{i,k\} \subseteq A} T(\beta_{ik}^*) + \cdots,   \, \sigma_A^2 \big),
\]
subject to appropriate structural constraints such as monotonicity, sparsity, and normalization.  Such a framework   would simultaneously correct for non-Gaussian payoff distortions via~$T$ and capture genuine synergistic effects via the interaction coefficients. This represents a promising yet computationally demanding direction for future research.
\\

As noted by a reviewer, the methodology (e.g., \eqref{TShap}) extends from the squared-loss regression setting to a generalized linear model framework \citep{mccullagh1989generalized}. In particular, the loss in \eqref{functional-shapley} or \eqref{spaiso-shapley} is now rephrased as
$$
\sum_{A\in 2^F}w_{\text{\tiny SH}}(A) L ( \delta_A; T(\nu_A )), \quad  \delta_A=\sum_{j\in A} T(\beta_j), \quad  L(\eta, y) =   -\eta y + b(\eta),
$$
where $b$ is the log-partition function in the specified  natural exponential distribution and  $(b')^{-1}$ induces the   canonical link   \citep{hardin2018glmext}. The resulting optimization subject to the same monotonicity constraints   and   sparsity/normalization constraints   as in \eqref{spaiso-shapley} may be termed as the sparse isotonic Shapley GLM. When $b(\eta) = \eta^2/2$,  this reduces to the squared-loss formulation analyzed in previous sections.

An appealing fact is that  the overall computational framework in Section~\ref{sec:comp} carries over  seamlessly.  The update of   $t$ (and hence $\delta$) is obtained by an order-restricted fit using the isotonic machinery adapted to the GLM loss; see, e.g.,  \cite{JSSv032i05} and \cite{LussRosset2014} for related PAVA and active-set extensions.    The update of $\gamma$ uses exactly the same surrogate function as in \eqref{surrofunc} with   the working   loss  $l(\gamma) = \sum_{A\in 2^F}w_{\text{\tiny SH}}(A) L ( \delta_A; T(\nu_A ))$, for which a direct calculation yields    $\nabla l(\gamma)= Z^\top W(b'(  Z\gamma)   - t) $,      recovering  \eqref{gradform} in the Gaussian case. The nonlinear operators from Theorem~\ref{th:normHT} are unchanged and Theorem~\ref{th:gammaoptalg} continues to apply with a distribution-specific curvature bound. For example, in logistic regression where
$b(\eta) = \log(1+\exp(\eta))$, the lower bound on  $\rho$ becomes  $\|Z^\top W  Z\|_2/4 $ (since $b''\le 1/4$). Derivations follow Section~\ref{sec:comp} closely and are omitted.

On the other hand, the Shapley interpretability is most natural in the Gaussian working model; moving to non-Gaussian families weakens the link to Shapley's foundational axioms. Also, in our experience, practical payoff functions rarely exhibit count-type behavior (as in Poisson models) or binary-valued outcomes (as in Bernoulli models), and we have not identified a compelling application in these settings. Therefore, although the GLM extension is methodologically sound and straightforward to implement, we leave its detailed investigation to future work.

\section{Conclusion}\label{sec:summ}
This paper   introduced \textit{Sparse Isotonic Shapley Regression} (SISR) to address two pressing limitations of Shapley values in the  context of XAI: the blind application of  an additive main-effect valuation to real-world payoff constructions driven by non-Gaussian distributions or domain-specific loss scales,  and the lack of native sparsity control in high-dimensional attribution tasks.

 Rather than abandoning the interpretability provided by a simple, additive structure of individual feature contributions,  SISR is designed to restore it.  Our framework achieves  nonlinear explainability by jointly learning a data-driven monotonic transformation via weighted isotonic regression   and enforcing an \(\ell_{0}\) sparsity constraint through normalized hard-thresholding.  Remarkably, the monotonic ordering constraint allows SISR to bypass a closed-form specification of the transformation, and sparsity both improves interpretability and accelerates computation in large feature spaces.  Each alternating step admits a closed-form update, enjoys global convergence guarantees, and is straightforward to implement.

Our $T$-additive Shapley model is built upon this novel capability of ``learning to be additive,''   which is able to simultaneously  recover the true payoff transformation and sparse support.  We reveal for the first time that the mere existence of irrelevant features or inter-feature dependence can induce a  payoff transformation that departs substantially from linearity, highlighting the necessity of nonlinear explainability.  Extensive experiments  validate that     SISR stabilizes attributions across radically different payoff constructions, correctly filters out spurious features, and aligns with established diagnostics, whereas standard Shapley values suffer severe rank and sign distortions.

By unifying domain adaptation and sparsity pursuit within the Shapley framework, SISR advances the frontier of {nonlinear explainability}, providing a theoretically grounded, robust, and scalable attribution methodology.

\bibliographystyle{apalike}
\bibliography{shapley}

\end{document}